\documentclass{article}
\PassOptionsToPackage{numbers, compress, authoryear}{natbib}
\usepackage[a4paper, margin=1.2in]{geometry}
\usepackage{amsmath}
\usepackage{amssymb}
\usepackage{amsthm}
\usepackage{bm} 
\usepackage{libertine}          
\usepackage{microtype}          

\usepackage[utf8]{inputenc} 
\usepackage{hyperref}       
\usepackage{url}          
\usepackage{booktabs}       
\usepackage{amsfonts}       
\usepackage{nicefrac}       
\usepackage{microtype}      
\usepackage{enumerate}
\usepackage{graphicx}
\usepackage{siunitx}
\usepackage{tikz-cd}
\usepackage{color}
\usepackage{subfigure}		
\usepackage{pgfplots}		
\pgfplotsset{compat=1.18}
\usepackage{tikz}
\usetikzlibrary{arrows}
\usepackage{mathabx}
\usepackage{stmaryrd}
\usepackage{mathtools}
\usepackage[most]{tcolorbox}

\DeclarePairedDelimiter{\abs}{\lvert}{\rvert}
\DeclarePairedDelimiter{\lnorm}{\lVert}{\rVert_{1}}

\usepackage[tableposition=top]{caption}
\usepackage{booktabs,dcolumn}

\definecolor{BlueViolet}{rgb}{0.54, 0.17, 0.89}

\newtheorem{theorem}{Theorem}[section]

\newtheorem{lemma}[theorem]{Lemma}

\newtheorem{definition}[theorem]{Definition}
\newtheorem{remark}[theorem]{Remark}

\newcommand\R{\mathbb{R}}

\newcommand\mcF{\mathcal{F}}
\newcommand\mcG{\mathcal{G}}

\newcommand\mcL{\mathcal{L}}
\newcommand\mcN{\mathcal{N}}

\newcommand{\dt}{\Delta t}

\newcommand\bij{\beta_{ij}}

\title{Automated discovery of finite volume schemes using Graph Neural Networks}

\author{
Paul Garnier\thanks{Corresponding author: paul.garnier@minesparis.psl.eu}\\ Jonathan Viquerat\\ Elie Hachem
\\
\small MINES Paris, PSL University \\
\small CEMEF - Centre for material forming
}
\date{}

\let\underbrace\LaTeXunderbrace

\begin{document}

\maketitle

\begin{abstract}
Graph Neural Networks (GNNs) have deeply modified the landscape of numerical simulations by demonstrating strong capabilities in approximating solutions of physical systems. However, their ability to extrapolate beyond their training domain (\textit{e.g.} larger or structurally different graphs) remains uncertain. In this work, we establish that GNNs can serve purposes beyond their traditional role, and be exploited to generate numerical schemes, in conjunction with symbolic regression. First, we show numerically and theoretically that a GNN trained on a dataset consisting solely of two-node graphs can extrapolate a first-order Finite Volume (FV) scheme for the heat equation on out-of-distribution, unstructured meshes. Specifically, if a GNN achieves a loss $\varepsilon$ on such a dataset, it implements the FV scheme with an error of $\mathcal{O}(\varepsilon)$. Using symbolic regression, we show that the network effectively rediscovers the exact analytical formulation of the standard first-order FV scheme. We then extend this approach to an unsupervised context: the GNN recovers the first-order FV scheme using only a residual loss similar to Physics-Informed Neural Networks (PINNs) with no access to ground-truth data. Finally, we push the methodology further by considering higher-order schemes: we train (i) a 2-hop and (ii) a 2-layers GNN using the same PINN loss, that autonomously discover (i) a second-order correction term to the initial scheme using a 2-hop stencil, and (ii) the classic second-order midpoint scheme. These findings follows a recent paradigm in scientific computing: GNNs are not only strong approximators, but can be active contributors to the development of novel numerical methods.

\end{abstract}

\section{Introduction}
\label{sec:intro}

\paragraph{Simulation as a grand challenge}
From weather forecasting and climate science to aeronautical design and the discovery of new materials, the reliable prediction of \emph{continuous} physical processes is a key challenge for Artificial Intelligence. Usually, simulating such physics involves solving partial differential equations (PDE) over complex domains represented as unstructured meshes \cite{HACHEM20108643}. Classical solvers such as the finite-volume method (FVM) remain a very strong strategy, yet deriving high-order, geometry-aware schemes that scale on modern hardware is not a straightforward task. 
Such methods are usually computationally intensive, and accurate simulations of realistic physical problems can typically require tens of thousands of core-hours on distributed architectures.
Moreover, each new simulation is performed independently, disregarding insights gained from previous runs, which motivates the integration of machine learning (ML) techniques for physics simulation. \cite{Lam2023GraphCast,Pathak2022FourCastNet,Bi2023Pangu, garnier2025meshmask}

\paragraph{Can neural networks inherit the guarantees of classical solvers?}
To address this, graph neural networks (GNNs) emerged as a natural fit for unstructured meshes, with message-passing architectures (MPS) offering a strong analogy with traditional PDE solvers. An MPS GNN will update at each layer the embedding of each node, using the current embedding of said node, and its edges. At the same time, each edge sees its embedding being updated using the nodes it is related to. This allows the information to flow from one node to another at each step of the network. More precisely, we define one update step as:

\begin{equation} \label{eq:mps_base}
h \leftarrow h + f^{\text{up}} \left(h, \displaystyle\sum_{\tilde{h} \in \mcN(h)} f^{\text{agg}}(e, h,\tilde{h})\right)
\end{equation}

\noindent where the functions $f$ are neural networks, $h$ is a node embedding, $\mcN(h)$ is the set of neighbours for the current node, and $e$ is the edge between the current node and one of its neighbour. In comparison, a first-order FV scheme solves the heat equation by updating the temperature at each step using:
 
\begin{equation} \label{eq:fvs}
T \leftarrow T+\displaystyle\sum _{\tilde{T} \in \mcN(T)}\phi(T, \tilde{T}) (\tilde{T}-T)
\end{equation}
 
 \noindent where $T$ is the current temperature, $\tilde{T}$ the temperature of a neighbouring cell in the mesh, and $\phi$ a geometric-based function that depends on each cell. By comparing updates (\ref{eq:mps_base}) and (\ref{eq:fvs}), it is easy to see how the message-passing paradigm mirrors the local flux exchanges of conservative PDE discretisations. Recent empirical evidence shows that carefully constructed GNNs can approximate time-evolution operators, accelerate implicit schemes, and act as learned preconditioners \cite{Sanchez2018PhysicsEngine,Pfaff2021MeshGraph,Battaglia2018Relational, garnier2024multi}. Yet two fundamental questions remain open:
\begin{itemize}
    \item Can GNNs \emph{extrapolate} to meshes, timesteps, or boundary conditions that lie far outside the training distribution?
    \item Can we \emph{interpret} what the network has learned from the perspective of classical FV theory?    
\end{itemize}

\paragraph{Neural algorithmic alignment meets scientific computing}
To provide answer to these two questions, we draw inspiration from the emerging concept of \emph{neural algorithmic alignment}: the idea that a network which \emph{structurally} approaches a classical algorithm will inherit its generalisation properties. While prior work has primarily investigated alignment with dynamic programming algorithms on combinatorial graphs, we extend the concept to \emph{continuous} physics. Specifically, we design a sparse message-passing GNN that precisely learns a first-order finite-volume scheme. 
Training on a set of 2-cells simulations is enough to recover the flux coefficients of the reference solver with provable accuracy: \textbf{if the empirical loss is $\varepsilon$, the error of the induced numerical method is bounded by $\mathcal{O}(\varepsilon)$ on \emph{any} unseen mesh.} Notably, the majoration constant depends on $\dt ^{-1}$ \cite{Velickovic2021NAR,Velickovic2019NEGA,Dudzik2022DynamicProg,Nerem2025ShortestPath}.

\paragraph{PINNS and symbolic regression} 
A complementary approach is to embed physical knowledge directly into the learning objective. Physics-Informed Neural Networks (PINNs) aim at minimizing PDE residuals \cite{Raissi2019PINN,Karniadakis2021PINNReview, lannelongue:hal-04750792}, while operator-learning approaches such as the Fourier Neural Operator (FNO) and its graph-based variants achieve mesh-independent generalisation by learning mappings between infinite-dimensional function spaces \cite{Li2020FNO,Li2020MGNO,Kovachki2023NeuralOperator,Lu2021DeepONet}.
Physics-informed GNNs extend these ideas to unstructured meshes, enforcing discrete conservation laws via finite-volume residuals \cite{Li2024FVGN,Li2025FVI,Salle2024Perf}. While PINNs can improve the performancess and create more physically-accurate outputs, models trained in this fashion are usually considered as complete black boxes. In parallel, advances in symbolic regression have shown that proper symbolic representations could be distilled directly from data \cite{Schmidt2009SR,Udrescu2020AIFeynman}. Using specific structures such as GNNs, recent work have been able to re-discover physical laws, and even discover new ones \cite{Cranmer2020Symbolic,Lemos2022Orbit}, suggesting a pathway toward models that both \emph{simulate} and \emph{explain}\footnote{something that would probably not have been possible without a GNN structure, given that a large modern transformers can't achieve such results \cite{vafa2025foundationmodelfoundusing}}. 

\paragraph{From approximation to automated discovery} Similar to other approaches where a PINN is used to discover unknown solutions \cite{wang2023asymptoticselfsimilarblowupprofile,kumar2024investigatingabilitypinnssolve,wang2025highprecisionpinnsunbounded}, our construction unlocks a path toward \emph{self-supervised discovery of numerical methods}. By searching in a richer parameter space and increasing the GNN's hop size or number of layers, the same alignment principles allows us to rediscover high-order schemes.

\paragraph{Contributions and outline.}

In this paper, we showcase a strong connection between GNNs and traditional FV methods. Our contributions are threefold. 

\begin{tcolorbox}[colframe=blue!60, colback=BlueViolet!5, boxrule=0.5pt, arc=5pt]

\begin{enumerate}
    \item First, we show numerically and theoretically that a GNN trained in a supervised manner on a minimal dataset of two-cell graphs can learn the fundamental principles of the heat equation and generalize to solve it on entirely out-of-distribution unstructured meshes. By applying symbolic regression to the components of the trained GNN, we show that the network learns the \textit{exact} corresponding first-order FV scheme, effectively rediscovering the known numerical method;
    \item Second, we discover that the same architecture trained using a physics-informed loss, without access to any ground-truth data, also exactly rediscovers the first-order scheme;
    \item Finally, we train two GNNs (one with features from a 2-hop neighborhoods, one with 2 message passing aggregation layers) using the same physics-informed loss to rediscover (i) a second-order scheme with a correction term exploiting the 2-hop stencil, and (ii) an explicit-midpoint scheme, similar to a second-order Runge-Kutta scheme.
\end{enumerate}

Those discovered schemes are more accurate than the standard first-order method, showcasing the potential of GNNs not only as solvers, but also as instruments for scientific discovery.

\end{tcolorbox}

The remainder of this paper is structured as follows. In \autoref{sec:fvm}, we provide a brief overview of the FV method for the heat equation. In \autoref{sec:gnn}, we detail the message-passing GNN architecture used in our experiments. Section \ref{sec:toy} presents our foundational result on the out-of-distribution generalization capabilities of GNNs. In particular, we demonstrate that the GNN learns the exact analytical form of the Finite Volume Method update rule via symbolic regression in \autoref{sec:regsymb}, using both a supervised loss and a unsupervised physics informed one. Finally, in \autoref{sec:gvm-higher}, we present our main result on the discovery of a novel, higher-order numerical scheme.

\section{Numerical framework}
\label{sec:fvm}

We aim to study how the temperature of a single-phase fluid may evolve over time and space.
We limit ourselves to square domains $\Omega \subset \mathbb{R}^2$ for simplicity, but the entire theory is easily expandable in three dimensions.

\subsection{The heat equation}

We consider a solid with density $\rho$, specific heat capacity $c$, and thermal conductivity $k$.
The distribution in space and evolution in time of the temperature field $T(x,y,t)$ is given by the heat equation:

\begin{equation}
  \rho c\, \frac{\partial T}{\partial t} = \nabla \cdot ( k\, \nabla T ) + Q(x,y,t)\,
  \label{eq:heat}
\end{equation}

\noindent where $Q$ is a source term. For simplicity, we assume the fluid is homogeneous, meaning its physical properties do not vary in space. Consequently, the density $\rho$, specific heat capacity $c$, and thermal conductivity $k$ are treated as constants throughout the domain. This allows the divergence term $\nabla \cdot (k \nabla T)$ to be simplified to $k \nabla^2 T$.
In the rest of this paper, we will consider the source term as constant accross time, and we will define the thermal diffusivity $\alpha = \frac{k}{\rho c}$ and the normalized source term $S = \frac{Q}{\rho c}$.
This leads to the following formulation:

\begin{equation}
  \frac{\partial T}{\partial t} = \alpha \nabla^2T + S(x,y)
  \label{eq:heat-main}
\end{equation}

\noindent We also define initial conditions $T(x,y, 0) = T_0(x,y)$, and  Dirichlet boundary conditions on $\partial \Omega$ set at 0.
\subsection{Finite volume method}

The finite volume method (FVM) is a widely used numerical technique for solving PDEs, such as the heat equation (\ref{eq:heat-main}).
The method starts by partitioning the domain $\Omega$ into a finite number of non-overlapping control volumes (or cells), denoted by $\Omega_P$.
We define $V_P$ as the volume (or area in 2D) of cell $\Omega_P$.
We associate a computational node (typically the cell centroid) $\mathbf{x}_P$ with each cell.
We also need to define a time discretization, and thus split our interval $[0, +\infty[$ into a sequence $[t^n, t^{n+1}]$, where $t^{n+1} = t^n + \Delta t$ and $t_0 = 0$ where $\Delta t$ is chosen to satisfy the diffusion stability constraint.
We can now integrate (\ref{eq:heat-main}) over each cell $\Omega_P$, and in an interval $[t^n, t^{n+1}]$:

\begin{equation}
  \int_{t^n}^{t^{n+1}} \int_{\Omega_P} \frac{\partial T}{\partial t} \, dV dt = \int_{t^n}^{t^{n+1}} \int_{\Omega_P} \alpha \nabla^2T dV dt + \int_{t^n}^{t^{n+1}} \int_{\Omega_P} S \, dV dt
  \label{eq:init-fvm}
\end{equation}

\noindent Applying the divergence theorem yields:

\begin{equation}
  \int_{t^n}^{t^{n+1}} \int_{\Omega_P} \frac{\partial T}{\partial t} \, dV dt = \int_{t^n}^{t^{n+1}} \int_{\partial \Omega_P} \alpha \nabla T \cdot \mathbf{n} \, dS + \int_{t^n}^{t^{n+1}} \int_{\Omega_P} S \, dV dt
  \label{eq:2fvm}
\end{equation}

\noindent where $\partial \Omega_P$ is the boundary of cell $\Omega_P$ and $\mathbf{n}$ is the outward unit normal vector.
The boundary $\partial \Omega_P$ consists of faces $f$ (or edges in 2D) shared with neighboring cells of $\Omega_P$, denoted $\mcN(\Omega_P)$.
At this stage, many different methods can be used to numerically solve equation (\ref{eq:2fvm}). In the context of traditional finite volumes, the temperature $T_P$ is considered constant inside each cell $\Omega_P$ (and similarly $S_P(x,y)$), and we define $T_P$ as the mean temperature inside $\Omega_P$.
We also define $A_f$ as the area (or length in 2D) of face $f$ between two cells, and $\delta_{PN} = ||\mathbf{x}_N - \mathbf{x}_P||$ the distance between the centroids of cells $N$ and $P$. This leads to the following approximation for a given face $f$\footnote{While in theory we have $\int_f (\alpha \nabla T) \cdot \mathbf{n}_f \, dS \approx \alpha \frac{A_f}{\delta_{PN}} (T_N - T_P)$, we use $=$ instead of $\approx$ in the remaining of the paper for the sake of simplicity.}:

\begin{equation}
\int_f (\alpha \nabla T) \cdot \mathbf{n}_f \, dS = \alpha \frac{A_f}{\delta_{PN}} (T_N - T_P)
\end{equation}

\noindent Summing over faces shared with neighboring cells $N$, we obtain:

\begin{equation}
  \int_{\partial \Omega_P} \alpha \nabla T \cdot \mathbf{n} \, dS = \sum_{N \in \mcN(\Omega_p)} \alpha \frac{A_{f_{PN}}}{\delta_{PN}} (T_N - T_P)
\end{equation}

\noindent Similarly, for the source term:

\begin{equation}
\int_{t^n}^{t^{n+1}} \int_{\Omega_P} S \, dV dt = \Delta tV_PS_P
\end{equation}

\noindent Finally, we consider a simple forward Euler scheme for time integration, \textit{i.e.}:

\begin{equation}
\int_{t^n}^{t^{n+1}} \int_{\Omega_P} \frac{\partial T}{\partial t} \, dV dt = V_P(T_P^{n+1} - T_P^{n})
\end{equation}

\noindent This leads to the following update scheme:

\begin{equation} \label{eq:default-fvm}
    T_P^{n+1} = T_P^n + \frac{\Delta t}{V_P} \sum_{N \in \mcN(\Omega_P)} \alpha \frac{A_{f_{PN}}}{\delta_{PN}} (T_N^n - T_P^n) + \Delta 
t S_P
\end{equation}

\noindent This update rule corresponds to a first-order accurate FV discretizations in both time and space, which will be our reference, first-order scheme for the rest of this paper. In \autoref{sec:toy} and \autoref{sec:regsymb}, we use equation (\ref{eq:default-fvm}) as the update scheme to compute the new temperatures.
In \autoref{sec:gvm-higher}, we investigate schemes of higher order, extending $\mcN(\Omega_P)$ to more distant neighbors.

%
%
%

\section{Graph neural networks}
\label{sec:gnn}

We now consider a mesh as an undirected graph $\mathcal{G} = (\mathcal{V},\mathcal{E})$.
$\mathcal{V} = \{\mathbf{v}_i\}_{i=1:N}$ is the set of nodes, where each $\mathbf{v}_i \in \mathbb{R}^{p}$ represents the attributes of node $i$.
$\mathcal{E} = \{\left(\mathbf{e}_k, r_k, s_k\right)\}_{k=1:N^e}$ is the set of edges, where each $\mathbf{e}_k$ represents the attributes of edge $k$, $r_k$ is the index of the receiver node, and $s_k$ is the index of the sender node.
In the framework of FVM, we consider the dual graph of the mesh, where each node corresponds to the centroid of a cell $\Omega_P$, and edges represent the faces connecting two adjacent cells.
The features constituting the node and edge attributes, $\mathbf{v}_i$ and $\mathbf{e}_k$, will be defined for each specific problem.
\subsection{Message passing GNN}

We define a message-passing graph neural network as a succession of graph net blocks. In each block, edge and node features are updated sequentially:

\begin{equation}
  \begin{alignedat}{3}
    &\mathbf{e}_k' &&= f^{\text{agg}}(\mathbf{e}_k,\mathbf{v}_{r_k},\mathbf{v}_{s_k}) &&\hspace{1em} \forall k \in \mathcal{E}
    \label{eq:edge_model} \\
    &\bar{\mathbf{e}}_r' &&= \displaystyle\sum_{k \text{ s.t. } r_k=r} \mathbf{e}_k' &&\hspace{1em} \forall r \in \mathcal{V} \\
    &\mathbf{v}_r' &&= f^{\text{up}}(\mathbf{v}_r, \bar{\mathbf{e}}_r') &&\hspace{1em} \forall r \in \mathcal{V}
  \end{alignedat}
\end{equation}

\begin{figure}[!t]
  \centering
  \includegraphics[width=0.7\textwidth]{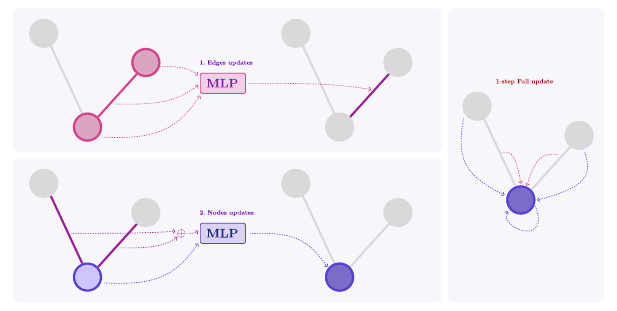}
  \caption{
    \textbf{(Top left)} Each edge is updated using its features and the features of its connected nodes.
\textbf{(Bottom left)} Each node is updated using its features and the aggregated features from its incoming edges.
\textbf{(Right)} A schematic of the information flow from a node's perspective after one message passing step.
}
  \label{fig:message-passing}
\end{figure}

\noindent The edge aggregation is represented here as a sum, but other operations like mean, min, or max are also possible.
The functions $f^{\text{agg}}$ and $f^{\text{up}}$ are typically Multi-Layer Perceptrons (MLPs).
The exact inputs to these MLPs can be simplified for specific tasks.
A visual overview of the process is proposed in \autoref{fig:message-passing}.
\begin{definition}[MPS]
  An L-layer Message Passing GNN (MPS) $\mathcal{M}_{\theta}$, with a set of weights $\theta$, computes for each layer $l \in \llbracket 1, L\rrbracket$:

  \begin{equation}
    h_r^l = f^{\text{up}} \Big(
    h_r^{l-1}, \displaystyle\sum_{k \text{ s.t. } r_k=r} f^{\text{agg}}(\mathbf{e}_k,h_r^{l-1},h_k^{l-1})
    \Big)
    \label{eq:mps}
  \end{equation}

 \noindent where $h_r^0 = \mathbf{v}_r$, and $f^{\text{up}}$ and $f^{\text{agg}}$ are $m$-layer MLPs with ReLU activations, except for the final layer.
\end{definition}

\noindent Given a graph $\mcG = (\mathcal{V},\mathcal{E})$, we define $\mcF$ as the function that outputs the temperature at the next timestep on graph $\mcG$.
For example, for a graph where $v_r = [T_r, \Delta t S_r, \Delta t/V_r]$ and the term for an edge between nodes $N$ and $P$ is $e_{NP} = \alpha \frac{A_{f_{PN}}}{\delta_{PN}} (T_N - T_P)$, the function $\mcF$ is defined for any node $r \in \mathcal{V}$ as:

\begin{align}
  \mcF(v_r) &=  v_{r,0} + v_{r,1} + v_{r,2}\sum_{p \in \mcN(\Omega_r)}e_{rp} \\
            &= T_r + \dt S_r + \frac{\dt}{V_r} \sum_{p \in \mcN(\Omega_r)}\alpha \frac{A_{f_{rp}}}{\delta_{rp}} (T_p - T_r)
\end{align}

\noindent Given a dataset of training graphs $\mathcal{D}_{\text{train}}$ and a $L$-layer MPS $\mathcal{M}_{\theta}$, we define the mean absolute error as:

\begin{equation}
  \mcL_{\text{MAE}}(\mathcal{M}_{\theta}, \mathcal{D}_{\text{train}}) = \frac{1}{\abs*{\mathcal{D}_{\text{train}}}}
  \sum _{\mcG \in \mathcal{D}_{\text{train}}}
  \frac{1}{\abs*{\mathcal{V}({\mcG})}}
  \sum _ {r \in \mathcal{V}({\mcG})} \lnorm*{
  h_r^L - \mcF(G)_r}
  \label{eq:mae_loss}
\end{equation}

\noindent We also define the physics-informed residual loss as:

\begin{equation}
  \mcL_{\text{PINN}}(\mathcal{M}_{\theta}, \mathcal{D}_{\text{train}}) = \frac{1}{\abs*{\mathcal{D}_{\text{train}}}}
  \sum _{\mcG 
\in \mathcal{D}_{\text{train}}}
  \frac{1}{\abs*{\mathcal{V}({\mcG})}}
  \sum _ {r \in \mathcal{V}({\mcG})} \lnorm*{
  \frac{h_r^L - h_r^0}{\Delta t} - S_r - \alpha\nabla^2h_r^L}
  \label{eq:pinn_loss}
\end{equation}

\noindent which is the residual of the governing PDE (\ref{eq:heat-main}) applied to each cell of the mesh.

\subsection{Symbolic regression for update and aggregation functions}

After training the GNN, we generate validation datasets $\mathcal{D}^{\text{agg}}$ and $\mathcal{D}^{\text{up}}$ containing input-output pairs for the MLPs $f^{\text{agg}}$ and $f^{\text{up}}$.
A symbolic regression algorithm then searches for analytical expressions, $f^{\text{agg}}_\mathrm{SR}$ and $f^{\text{up}}_\mathrm{SR}$, that accurately approximate the GNN's predictions while minimizing complexity.
The process explores a vast combinatorial space of mathematical formulas constructed from a predefined set of operators, variables, and constants.

In \autoref{sec:regsymb}, we use this approach after training a GNN in a supervised fashion to ensure that the model indeed learned an exact first-order FV scheme. 
In \autoref{sec:gvm-higher}, instead of relying on a known function $\mcF$, we train our model using a physics-informed loss and then employ symbolic regression to uncover the numerical scheme learned by the GNN.

\section{GNN extrapolates out-of-distribution data for the heat equation}
\label{sec:toy}

In this section, we first consider a simplified training dataset (both in terms of mesh and inputs), and train a GNN to learn a first-order scheme using a supervised loss. We demonstrate generalization results both numerically and theoretically on a simplified GNN architecture (1 layer and a few weights), before expanding these results to more complex model ($L$ layers of width $d$). By applying symbolic regression to the trained GNNs, we show that the first-order FV scheme is exactly recovered. In a second time, the training regime is switched from supervised to unsupervised using a PINN-like loss function. In this context, we show, again using symbolic regression, that the first-order FV scheme is exactly recovered.



\subsection{Training dataset}
\label{subsec:training-dataset}
To rigorously test the generalization capabilities of our GNN, we construct a minimal training dataset. The core principle is to expose the model to the fundamental physics of heat exchange on the simplest possible topology, therefore compelling it to learn the underlying mathematical structure rather than memorizing complex spatial arrangements.

\subsubsection{Graph structure and geometry}
Each training instance is a graph $\mathcal{G}$ consisting of only two nodes connected by a single edge. This graph represents the dual of a two-cell mesh, where each node corresponds to the centroid of a cell and the edge represents the shared face. To isolate numerical learning from geometric complexities, we fix the geometry of these cells to be identical equilateral triangles. This standardization ensures that all geometric factors in the FVM update rule (cell volume, face area, and inter-node distance) are constant across the training set. Specifically, we set the triangle side length $\delta=2$, which yields the following fixed parameters for each cell $i$:
\begin{itemize}
    \item Cell Area : $V_i = \sqrt{3}$
    \item Face Area : $A_{f_{ij}} = 2$
    \item Distance between centroids: $\delta_{ij} = 2/\sqrt{3}$
\end{itemize}
This configuration simplifies the geometric term $\frac{A_{f_{ij}}}{\delta_{ij}}$ to a constant value of $\sqrt{3}$. More importantly, it leads to $\frac{1}{V_i}\frac{A_{f_{ij}}}{\delta_{ij}} = 1$

\begin{remark}
The main goal of this configuration is to allow for algebraic simplifications on the training dataset. Indeed, on the training set, each update now becomes
    \begin{equation}
        T_P^{n+1} = T_P^n + \alpha\dt (T_N^n - T_P^n) + \Delta 
t S_P
    \end{equation}
\end{remark}

\subsubsection{Feature space}
The physical state of each two-cell system is defined by the temperature $T_i$ and the source term $S_i$ for each cell $i \in \{1, 2\}$. We generate our training data by sampling these values uniformly from a predefined range:
\begin{equation}
    T_i, S_i \sim U(0, \mathcal{T}_{\text{train}})
\end{equation}
where $\mathcal{T}_{\text{train}}$ is the maximum temperature used during training. We then normalize the temperature and the source term in $[-1,1]$ during training. We also use random boundary conditions (whether it is to activate them or not, or regarding their values).

\subsubsection{Dataset composition}
The final training dataset, $\mathcal{D}_{\text{train}}$, is composed of two components:
\begin{enumerate}
    \item \textbf{Random samples}: A set of graphs where $(T_1, T_2, S_1, S_2)$ are sampled randomly as described above.
    \item \textbf{Corner cases}: A small, deterministic set of graphs designed to constrain the learned function. These include graphs representing zero gradients, uniform temperatures, and zero source terms: $G(0,0,0,0)$, $G(1,1,0,0)$, $G(0,0,1,1)$, and $G(0,1,0,0)$.
\end{enumerate}
The ground truth for each graph $\mcG$ is the temperature at the next timestep, computed using the FV scheme from (\ref{eq:default-fvm}), denoted $\mcF(\mcG)$. The resulting dataset is thus a collection of pairs $(\mcG, \mcF(\mcG))$. This intentionally constrained setup, illustrated in \autoref{fig:dataset1}, provides the necessary information for the GNN to learn the local physics of heat transfer, which we hypothesize is sufficient for generalization to arbitrarily large and complex graphs.

\begin{figure}[!t]
  \centering
  \includegraphics[width=0.7\textwidth]{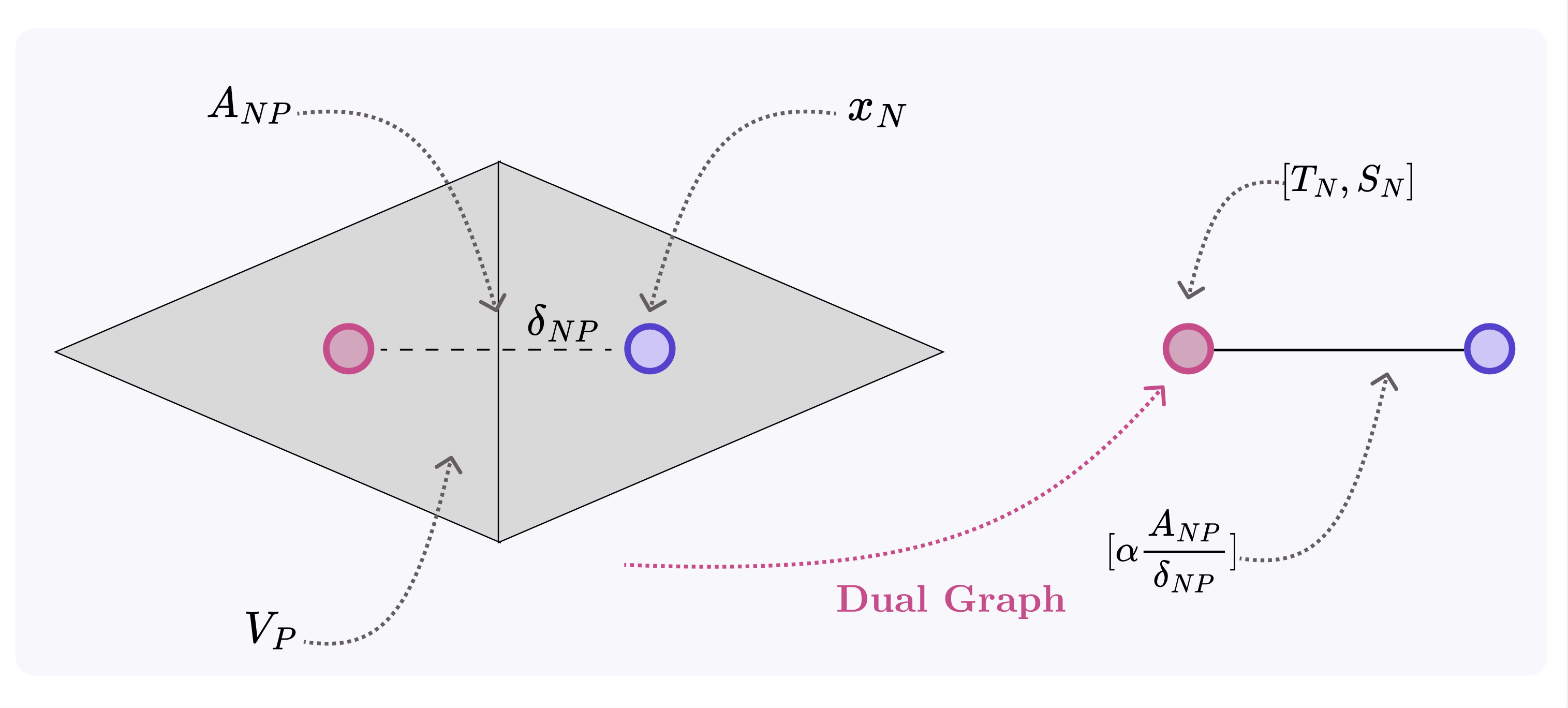}
  \caption{
    A graph from the training set made from a 2-cells mesh.
}
  \label{fig:dataset1}
\end{figure}

\subsection{Supervised generalization from two-node graphs}
\label{subsec:theo}
This section presents a foundational result demonstrating that a GNN, trained on a highly restricted dataset, can learn a governing physical law and generalize to vastly more complex, out-of-distribution (OOD) scenarios. The experiment is designed to build intuition for the next, more complex results.

\noindent We begin with a simple 1-layer GNN, $\mathcal{M}_\theta$, and scheme (\ref{eq:default-fvm}): $\mcF$. For simplicity of theoretical proofs, we set the node features to $v_r = [T_r, \Delta t S_r, \Delta t/V_P]$ and edge features to $e_{NP} = \alpha \frac{A_{f_{PN}}}{\delta_{PN}} (T_N - T_P)$\footnote{As a reminder, we use this set of very simplified inputs for the sake of theoretical proofs. In the next sections of the paper, we go back to a standard set of inputs.}. The GNN update is defined as:
$h_r^1 = f^{\text{up}}(v_{r,0}, v_{r,1}, v_{r,2}\bar{\mathbf{e}}_r')$, where $\bar{\mathbf{e}}_r'$ is the sum over aggregated edge messages from $f^{\text{agg}}(e_k)$.
As we demonstrate in \autoref{lemma:simpleweights} in the appendix, it is straightforward to show that with a simple set of weights (e.g., all weights and biases being 1 or 0), the GNN can perfectly replicate the FVM update step $\mcF$. Indeed, setting all biases to $0$, and the remaining weights to $1$ simply sums all features altogether, which does build the scheme (\ref{eq:default-fvm}). 

\noindent The key result of this section is a theorem showing that if the GNN's predictions are close to the true FVM update on a small training set, its error on \textit{any} graph remains bounded. We start with a result based on a simplified GNN:

\begin{theorem}[Out-of-distribution generalisation by simple GNN]
  Let $\varepsilon > 0$.
Let $\mathcal{M}_{\theta}$ be a 1-layer GNN such that for all graphs $\mcG$ in the training set $\mathcal{D}_{\text{train}}$, the error is bounded: $\lnorm*{ \mathcal{M}_{\theta}(\mcG)_r - \mcF(\mcG)_r} < \varepsilon$ for any node $r$.
Let $\mcN$ be the maximum number of neighbors a single cell can have in a well-built mesh (outside of the training set).
Then, for any graph $\mcG$ and any node $r \in \mathcal{V}({\mcG})$:

  \begin{equation}
    \lnorm*{ \mathcal{M}_{\theta}(\mcG)_r - \mcF(\mcG)_r} <  \varepsilon \left(4 + \frac{4}{\Delta t} + \mcN\right)
  \end{equation}

\end{theorem}

\noindent The intuition is that a sufficiently diverse but simple training set forces the GNN weights to align with the analytical solution, ensuring that the learned function is not just a fit to the training data but an accurate approximation of the underlying physical law. The proof is detailed in the appendix \ref{proof:th1}.

\subsection{Theoretical results}

We extend the previous finding to a general, deep GNN. With an appropriate training loss that combines the mean absolute error with a sparsity regularizer ($L_0$-norm), we can force the GNN to learn the most efficient, and thus physically correct, representation.
We define the training loss as:
\begin{equation}
  \mcL(\mathcal{M}_{\theta}, \mathcal{D}_{\text{train}}) = \mcL_{\text{MAE}}(\mathcal{M}_{\theta}, \mathcal{D}_{\text{train}}) + \eta\vert\vert\theta \vert \vert _0
  \label{eq:mae_l0_loss}
\end{equation}

\begin{tcolorbox}[colframe=blue!60, colback=BlueViolet!5, boxrule=0.5pt, arc=5pt]

\begin{theorem}[Out-of-distribution generalisation]
  \label{theorem:main}
  Let $\mathcal{M}_{\theta}$ be an $L$-layer GNN with $m$-layer ReLU MLPs of width $d$.
Let $\varepsilon > 0$, with $\varepsilon < \eta$.
If the trained model achieves a loss $\mcL(\mathcal{M}_{\theta}, \mathcal{D}_{\text{train}}) < \varepsilon$ on the two-node graph dataset, then for any graph $\mcG$, the model's prediction error is bounded:

  \begin{equation}
    \label{eq:main-theorem}
    \lnorm*{ \mathcal{M}_{\theta}(\mcG)_r - \mcF(\mcG)_r} < 2\varepsilon \left( 1 + \frac{2}{\Delta t}\right)
  \end{equation}
\end{theorem}

\end{tcolorbox}

This theorem establishes that a deep GNN, under sparsity pressure, does not simply memorize the training data but discovers the underlying structure of the SFVS operator. 
The sparsity forces the network to find the best parameterization, which corresponds to the exact scheme\footnote{and not an approximation that may be prone to explosion with unseen meshes.}. We demonstrate these results by first studying the optimal number of parameters needed and then by finding several inequalities based on the dataset and the Fourier condition in \autoref{appendix:main-demo}.

\subsection{Numerical results}
\label{subsec:exp}

We empirically validate \autoref{theorem:main} by training a 2-layer GNN ($d=32$) with and without $L_1$ regularization (a convex proxy for $L_0$, making it easier to optimize). All models are trained on the same two-node graph dataset and evaluated on the same out-of-distribution (OOD) and 50-step rollout tasks. Sparsity is a cornerstone of the proof, and this experiment demonstrates its practical importance. We evaluate performance using several metrics:

\begin{itemize}
  \item \textbf{Out-of-distribution MSE}: MSE loss on 100 large, unstructured graphs with approximately 300 cells, generated with GMSH \cite{Geuzaine2009}.
  \item \textbf{50-step rollout}: The mean squared error between a 50-step autoregressive prediction by the GNN ($\mathcal{M}_{\theta}^{\circ 50}(\mcG)$)\footnote{where $f^{\circ n}(x)$ is defined as $\underbrace{f \circ ... \circ f}_\text{$n$ times}(x)$} and the ground truth from the FVM solver ($\mcF^{\circ 50}(\mcG)$).
\end{itemize}

\begin{figure*}[!ht]
  \centering
  \includegraphics[width=1\textwidth]{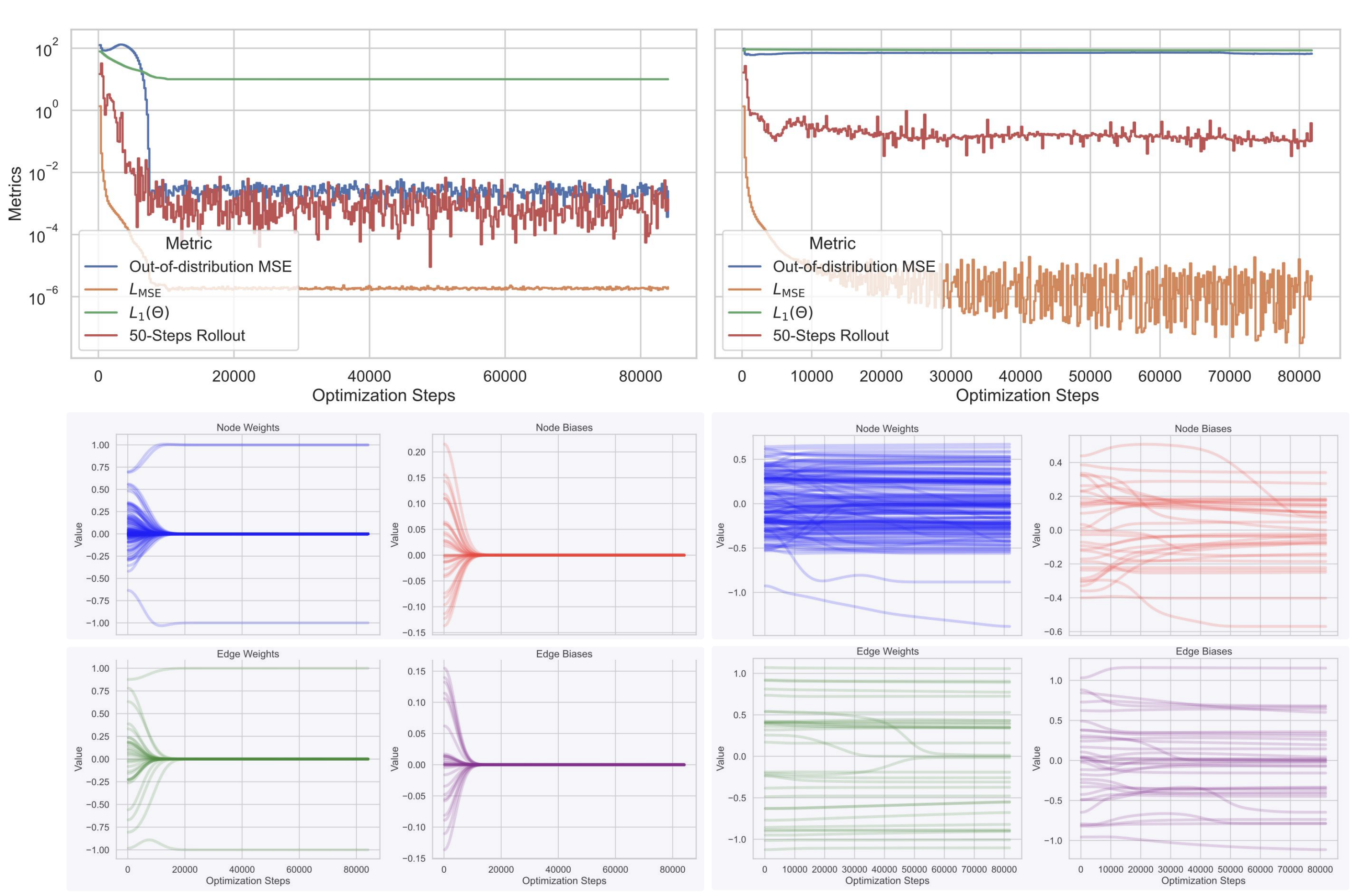}
  \caption{
We display of the \textbf{left} the results with regularization and on the \textbf{right} the results without it. \textbf{(Top)} Performance metrics for models trained with and without $L_1$ regularization. The regularized model (sparse GNN) shows significantly lower out-of-distribution MSE and 50-step rollout error compared to the non-regularized model, despite all models achieving near-zero training loss. \textbf{(Bottom)} Weight distributions for the final layer of the GNN's update MLP. The sparse model's weights are clustered around zero, while the dense model's weights are more widely distributed. This illustrates how $L_1$ regularization encourages sparsity, which is key for generalization.
}
  \label{fig:main-res}
\end{figure*}

\noindent Models were trained on a training dataset of 100 2-cell graphs, with a batch size of 4. We used 90$k$ training steps with a learning rate of $10^{-4}$ using the Adam optimizer \cite{kingma2017adam} and a sparsity constraint factor $\eta = 10^{-3}$. Results regarding the metrics and the model's weights are available in \autoref{fig:main-res}.

As we can see, both models perfectly fit the simple training data. However, only the sparse model generalizes: its OOD and rollout errors remain low, whereas the non-regularized model fails. This empirically confirms that sparsity is essential for the GNN to learn the underlying physical principle rather than overfitting to the training distribution. 

\subsection{A more challenging learning task}
\label{sec:regsymb}

In the previous section, we showed that a GNN can generalize from a set of simplified inputs. We now go a step further and demonstrate that it can learn the \textit{exact} analytical form of the finite-volume update rule from less direct inputs using symbolic regression.

To make the learning problem harder and test the GNN's ability to discover physical formulas, we provide the model with more primitive features, requiring it to learn the necessary geometric and physical relationships.
The node and edge features are now set to:

\begin{equation}
  \begin{alignedat}{2}
    &v_r &&= [T_r, S_r, \Delta t, V_r] \\
    &e_{NP} &&=  [T_N, T_P, A_{NP}, \delta_{NP}]
  \end{alignedat}
\end{equation}

\paragraph{} We increase the model's capacity to handle this more complex task, using a GNN with three hidden layers ($m=3$) of width $d=128$ for both $f^{\text{agg}}$ and $f^{\text{up}}$. We also improve their capacity\footnote{especially regarding multiplication} by going from a simple ReLU-activated MLP to a Gated MLP \cite{dauphin2017language} with GeLU non-linearity \cite{hendrycks2023gaussian} where each layer is defined as:

\begin{equation}
    \mathbf{X} = W_f\Big(\text{GeLU}\big(W_l \mathbf{X} + b_l\big) \odot (W_r \mathbf{X} + b_r)\Big) + b_f
\end{equation}

\noindent where $\odot$ is an Hadamard product. The model is trained on the same two-node graph dataset $\mathcal{D}_{\text{train}}$ with $L_1$ regularization.
After training, the learned functions $f^{\text{agg}}$ and $f^{\text{up}}$ are extracted and analyzed.

\paragraph{} We use symbolic regression, implemented with the \textit{PySR} library \cite{cranmer2023interpretablemachinelearningscience}, to find the mathematical equations that the trained GNN has learned.
We generate datasets of input-output pairs from the trained $f^{\text{agg}}$ and $f^{\text{up}}$ MLPs and search for the simplest analytical expressions that fit this data.
The search space is restricted to basic arithmetic operators (+, -, $\times$, /), and candidate equations are scored based on a combination of accuracy and complexity.

\subsection{Results}

Using the same dataset, Gated MLPs, and a much harder set of features, the symbolic regression successfully recovered the exact structure of scheme (\ref{eq:default-fvm}) (see \autoref{fig:sr-1}). The best-fitting equations were:

\begin{tcolorbox}[colframe=blue!60, colback=BlueViolet!5, boxrule=0.5pt, arc=5pt]

\begin{itemize}
    \item \textbf{Edge model ($f^{\text{agg}}$)}: For an edge with inputs $(T_N, T_P, A_{NP}, \delta_{NP})$, the discovered formula was:
    \[ f^{\text{agg}}_\mathrm{SR} = c_1 \cdot \frac{A_{NP}}{\delta_{NP}} (T_N - T_P) \]
    This exactly matches the heat flux term in the FVM formulation, with $c_1$ being a learned constant corresponding to the thermal diffusivity $\alpha$.

    \item \textbf{Node model ($f^{\text{up}}$)}: For a node with inputs $(T_r, S_r, \Delta t, V_r)$ and aggregated edge messages $\bar{e}'_r$\footnote{using a sum aggregation, as defined at the beginning}, the discovered formula was:
    \[ f^{\text{up}}_\mathrm{SR} = T_r + \frac{\Delta t}{V_r} \bar{e}'_r + \Delta t \cdot S_r \]
    which is precisely the update defined in equation (\ref{eq:default-fvm}).
\end{itemize}

\end{tcolorbox}
\begin{figure*}[!ht]
  \centering
  \includegraphics[width=1\textwidth]{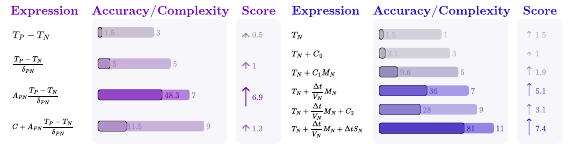}
  \caption{
    Symbolic regression results for the $f^{\text{agg}}$ function on the \textbf{left} and $f^{\text{up}}$ function on the \textbf{right}. The plots show the trade-off between equation complexity (number of operations and variables) and accuracy. The optimal equations correspond exactly to the terms in the FVM, demonstrating that the GNN learned the underlying physics.
  }
  \label{fig:sr-1}
\end{figure*}

\noindent These results provide strong evidence that the GNN did not merely approximate the function but learned the valid underlying physical equation. 
This success lays the foundation for our final experiment, where we apply this methodology to discover higher-order numerical schemes.

\begin{remark}
    Importantly, if we switch from using a supervised loss to a PINN with \autoref{eq:pinn_loss}, our model successfully learns the same first-order scheme. We confirm this by running the same symbolic regression as previously.
\end{remark}

\begin{tcolorbox}[colframe=blue!60, colback=BlueViolet!5, boxrule=0.5pt, arc=5pt]

We demonstrated both theoretically, and experimentally with pure performance metrics and symbolic regression, that a GNN can learn a first-order FV scheme using only a simplified training dataset. We also discovered that similar results are obtained using a PINN, thus training the GNN using only residual from the heat equation. 

We now make our architecture more complex to see if a GNN can discover higher-order schemes.

\end{tcolorbox}

\section{GNN discovers higher-order finite volume schemes}
\label{sec:gvm-higher}

Having established that a GNN can learn and generalize known physics, we now tackle a more ambitious goal: discovering higher-order numerical schemes. We train a GNN on a physics-informed loss without a direct supervision signal from a known solver and then use symbolic regression to interpret the learned update rule. More importantly, we will expand the receptive field of said GNN by either:

\begin{itemize}
    \item having multiple message passing layers, thus aggregating information from nodes up to a distance $L$, where $L$ is the number of layers
    \item aggregating information from 1-hop neighbors, like before, but also from 2-hop neighbors
\end{itemize}

\subsection{Experimental setup}
\paragraph{Dataset} We generate a new dataset of 100 unstructured 2D meshes within a unit square. Each mesh contains between 100 and 600 cells. Initial and boundary conditions are randomized for each simulation to ensure diversity. 
We also use two different sorts of meshes: with regular patterns, and with irregular patterns. The regular meshes are defined as an unstructured grid with regular cell shape. In contrast, the unstructured meshes are defined as an unstructured grid with two random attraction points that let two areas of the unit square be meshed with smaller elements. This leads to non-constant geometrical features across the cells. Samples from the dataset are presented in figure \ref{fig:mesh_dataset}.

\begin{figure}
\centering
\subfigure[Regular mesh]{
\includegraphics[width=.35\textwidth]{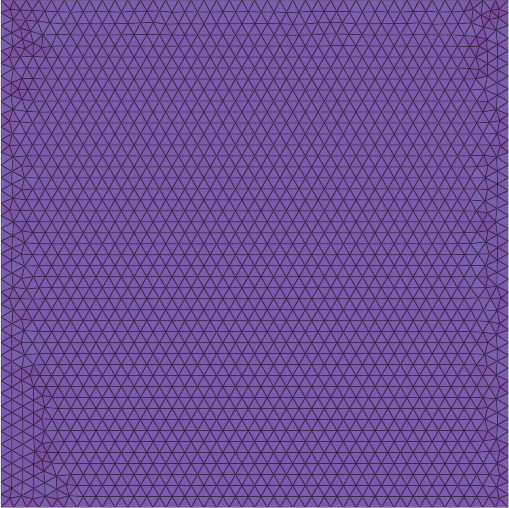}
} \qquad
\subfigure[Irregular mesh]{
\includegraphics[width=.35\textwidth]{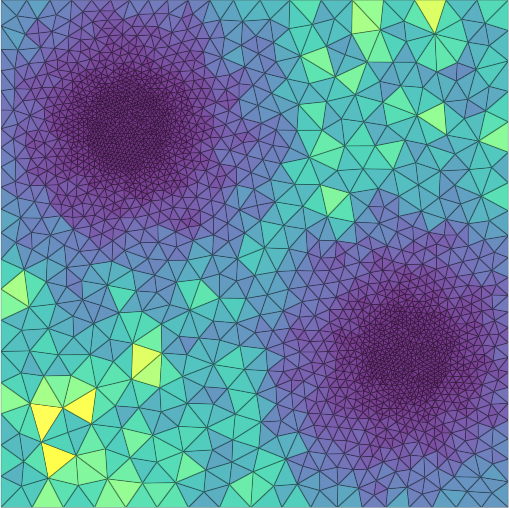}
}
\subfigure{
  \begin{tikzpicture}
    \begin{axis}[   
    hide axis, scale only axis, height=0pt, width=0pt, colormap/viridis,
    colorbar horizontal, point meta min=0, point meta max=1,
    colorbar style={width=6cm, height=0.25cm, xtick={0,1}}]
          \addplot [draw=none] coordinates {(0,0)};
    \end{axis}
  \end{tikzpicture}
}
\caption{Sample meshes from the regular (left) and irregular (right) datasets. The colorscale indicates the cell determinants, normalized to the maximal determinant size observed on the irregular mesh.}
\label{fig:mesh_dataset}
\end{figure}

\paragraph{New GNN architecture} To allow the model to learn a higher-order scheme, which requires information from more distant neighbors, we modify the GNN architecture. In addition to aggregating messages from the immediate 1-hop neighborhood $\mcN_1(\Omega_P)$, we add a second aggregation module for the 2-hop neighborhood $\mcN_2(\Omega_P)$. The node update is now a combination of these two aggregations:
\begin{equation}
    h_P' = f^{\text{up}} \left( \underbrace{v_P}_\text{\color{BlueViolet}purple features\color{black}}, \underbrace{\sum_{N \in \mcN_1(\Omega_P)} f^{\text{agg}_1}(e_{PN})}_\text{\color{purple}pink features\color{black}}, \underbrace{\sum_{Q \in \mcN_2(\Omega_P)} f^{\text{agg}_2}(e_{PQ})}_\text{\color{teal}green features\color{black}} \right)
\end{equation}
where $f^{\text{agg}_1}$ and $f^{\text{agg}_2}$ are separate MLPs for 1-hop and 2-hop edges, respectively. An overview of this new architecture is described in \autoref{fig:2hops}. We also continue to use Gated MLPs instead of regular MLPs.

\begin{figure}[!t]
  \centering
  \includegraphics[width=1\textwidth]{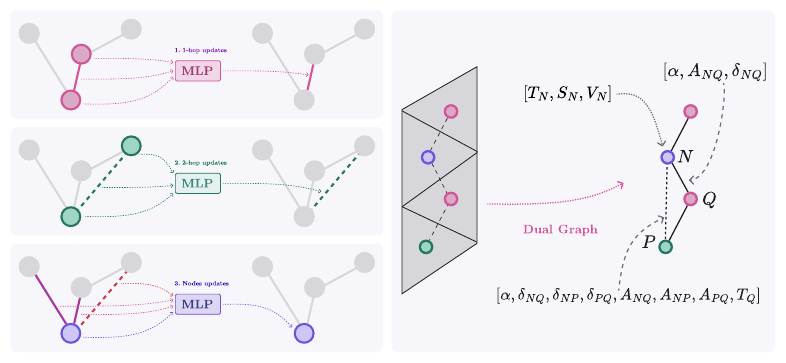}
  \caption{
    Illustration of the 2-hop neighborhood used for the GNN architecture in the higher-order scheme discovery task.
    The central node (in \color{BlueViolet}purple\color{black}) receives messages not only from its direct neighbors (in \color{purple}pink\color{black}) but also from its neighbors-of-neighbors (in \color{teal}green\color{black}). This allows the GNN to learn more complex interactions and approximate higher-order derivatives.
}
  \label{fig:2hops}
\end{figure}

\begin{remark}
The architecture used here, where the second hop is accessed directly, instead of through the indirect propagation of information with a 2-layers GNNs is very similar to the $K$-hop architecture defined in \cite{feng2023powerfulkhopmessagepassing}. 
Not only this architecture allows us to perform 2 different symbolic regressions on the 2 aggregators, but it is also theoretically stronger than the regular 1-hop message passing architecture, independently of the number of layers.
With a standard architecture and 2 pairs of $f^{\text{up}}$ and $f^{\text{agg}}$ functions, it is much more tedious to reconstruct how features are propagated.
\end{remark}

\noindent We also re-use the former architecture but with two layers of message passing. This allows the model to implicitly aggregate information from the 2-hop neighborhoods by using temperatures from the first layer in the second layer. Our intuition is that such architecture is well-suited to discover a semi-implicit scheme. We define the first architecture as \textit{GNN-2ndOrder} and the second one as \textit{GNN-Midpoint}.

\paragraph{Model inputs} The inputs for the 1-hop neighbors and the nodes are the same as in \autoref{sec:regsymb}. For the 2-hop neighbors, we provide all available features, including temperatures, source terms, and geometric properties of the nodes and the path connecting them (using 1-hop edges). If a 2-hop node can be accessed through multiple 1-hop nodes, we average the features through all possible paths.

\paragraph{Training} The GNN is trained using only the physics-informed loss $\mcL_{\text{PINN}}$ (\autoref{eq:pinn_loss}), which measures the residual of the heat equation. This means we use no ground-truth data, and that the GNN must learn a valid time-stepping scheme on its own.

\subsection{Results: discovery of two new schemes}

As mentioned earlier, we trained two different GNNS. Both GNNs successfully converged, achieving significantly lower residual losses than the standard first-order FV scheme, suggesting they learned a more accurate method. We then applied symbolic regression to the learned MLPs ($f^{\text{agg}_1}$, $f^{\text{agg}_2}$, and $f^{\text{up}}$) for the first GNN and $f^{\text{agg}_1}$, $f^{\text{agg}_2}$, $f^{\text{up}_1}$ and $f^{\text{up}_2}$ for the second one. 

\begin{tcolorbox}[colframe=blue!60, colback=BlueViolet!5, boxrule=0.5pt, arc=5pt]

The analysis revealed that both GNNs had discovered higher-order correction terms. The two learned update rules are:

\begin{equation}
    \begin{aligned}
      T_P^{\,n+1}
      \;=\;&
      \underbrace{T_P^{\,n} \;+\;\Delta t\,S_P}_\text{\color{BlueViolet}Node Features\color{black}}
      \;+\;
      \frac{\Delta t}{V_P}
      \Biggl[
        \underbrace{%
          \sum_{N\in\mcN_1(\Omega_P)}
            \alpha\,\frac{A_{f_{PN}}}{\delta_{PN}}
            \bigl(T_N^{\,n}-T_P^{\,n}\bigr)}_{%
          \text{\color{purple}1st-hop contribution\color{black}}}
      \\
      &\hphantom{T_P^{\,n}
        + \frac{\Delta t}{V_P}\Bigl[}\;+\;
        \underbrace{%
          \sum_{Q\in\mcN_2(\Omega_P)}
            \frac{1}{2}\alpha\,\,
            \bigl(T_Q^{\,n}-T_P^{\,n}\bigr)}_{%
          \text{\color{teal}2nd-hop contribution\color{black}}}
      \Biggr]
    \end{aligned}
  \label{eq:lsq-two-ring-update}
\end{equation}

\noindent for \textit{GNN-2ndOrder}, and

\begin{equation}
    \begin{aligned}
      T_P^{\,n+1}
      \;=\;&
      T_P^{\,n} \;+\;\Delta t\,S_P 
      \;+\;
      \frac{\Delta t}{V_P}
        \underbrace{%
          \sum_{N\in\mcN_1(\Omega_P)}
            \alpha\,\frac{A_{f_{PN}}}{\delta_{PN}}
            \bigl(T_N^{\,n+\frac{1}{2}}-T_P^{\,n+\frac{1}{2}}\bigr)}_{%
          \text{\color{teal}Second Layer contribution\color{black}}}
      \\
      T_P^{\,n+\frac{1}{2}} \;=\;&
      T_P^{\,n} \;+\;\Delta t\,S_P 
      \;+\;
      \frac{1}{2}
      \frac{\Delta t}{V_P}
        \underbrace{%
          \sum_{N\in\mcN_1(\Omega_P)}
            \alpha\,\frac{A_{f_{PN}}}{\delta_{PN}}
            \bigl(T_N^{\,n}-T_P^{\,n}\bigr)}_{%
          \text{\color{purple}First Layer contribution\color{black}}}
    \end{aligned}
  \label{eq:midpoint-update}
\end{equation}

\noindent for \textit{GNN-Midpoint}.

\end{tcolorbox}

\paragraph{Numerical analysis} Both models learned a non-trivial correction term based on 2-hop neighbors\footnote{An important distinction to make here is that \textit{GNN-2ndOrder} learns this correction with the direct 2nd hop, while \textit{GNN-Midpoint} learns it through the propagation of the 2-hop nodes inside 1-hop nodes with the two layers.}. This term acts as a discrete approximation to higher-order spatial derivatives, improving the accuracy of the simulation. 
We compare those two methods to the traditional first-order FV scheme (\ref{eq:default-fvm}) and a second-order Crank-Nicholson formulation:

\begin{equation}
  V_P\frac{T_P^{n+1}-T_P^{n}}{\Delta t}+\alpha\sum_{N\in\mathcal N_1(P)}\frac{A_{f_{PN}}}{\delta_{PN}}\Bigl(T_N^{[n+1/2]}-T_P^{[n+1/2]}\Bigr)=0,
  \label{eq:CN}
\end{equation}

\noindent with $T_P^{[n+1/2]}:=\tfrac12(T_P^{n+1}+T_P^{n})$.

\begin{figure}[!th]
  \centering
  \includegraphics[width=1\textwidth]{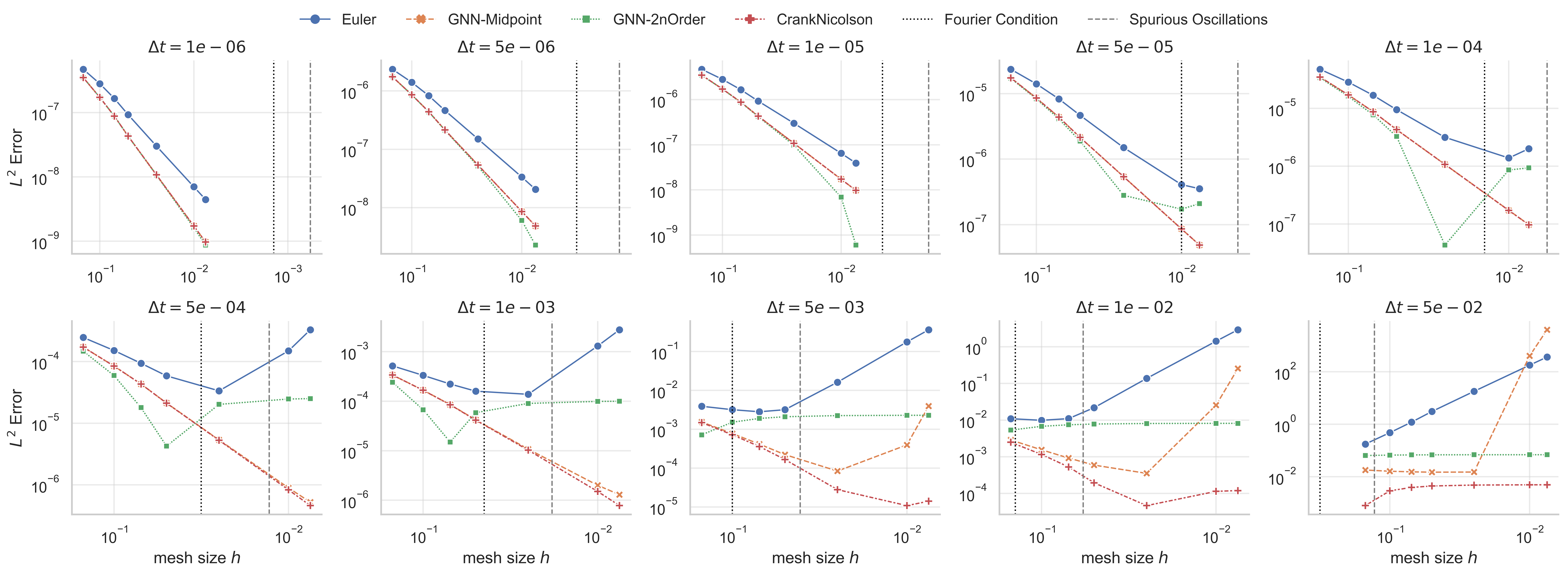}
  \caption{
    Residual for the four studied schemes studied, computed for different mesh sizes and timestep values. We also display the Fourier condition for each timestep, as well as the area where spurious oscillations might start even for Crank-Nicholson.
}
  \label{fig:convergence}
\end{figure}

\paragraph{}
This approach achieves higher accuracy by estimating solution gradients within each cell to better approximate the flux across cell faces. The GNN-discovered schemes, in contrast, do not explicitly compute gradients or apply limiters. Instead, it has learned a linear correction term based on a wider, 2-hop stencil. While all approaches leverage a wider neighborhood to approximate higher-order spatial derivatives, the GNN's methods are a direct, data-driven correction to the cell-average temperature update. In contrast, classical schemes focus on a more physically-grounded reconstruction of the intra-cell solution field. The learned term can be interpreted as a novel discrete operator that implicitly captures the necessary information for a higher-order update without the explicit nonlinear machinery of traditional approaches, either using wider hops or a semi-implicit scheme. More precisely, \textit{GNN-Midpoint} actually learns an explicit midpoint scheme (or second order Runge-Kutta)\footnote{In future work, it would be interesting to see if with 4 layers, a GNN can actually learn a full Runge-Kutta method.}.

\paragraph{}
To numerically verify the spatial second-order claim, we conduct a convergence study on a dataset of irregular meshes. This study is performed on a wide range of timestep values, in order to evaluate the stability and accuracy of the different schemes. Results are available in \autoref{fig:convergence}. Overall, we find that when the Fourier condition is met, both second-order GNN schemes display performances similar to that of the CN scheme. The \textit{GNN-Midpoint} scheme keeps following it for a large range of validity, while \textit{GNN-2ndOrder} struggles when we get close to the Fourier condition, even on very small time-step values. This result demonstrates that the combination of flexible GNN architectures and physics-informed training can be a powerful tool for discovering new numerical methods for solving PDEs. Importantly, we also demonstrate that while the performances might be equivalent or slightly lower, the discovered schemes offers a much lower memory consumption, see \autoref{fig:memory}.

\begin{figure}[!t]
  \centering
  \includegraphics[width=1\textwidth]{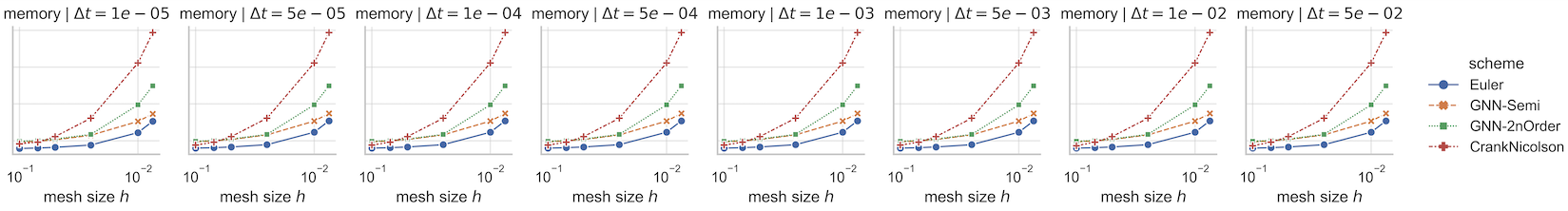}
  \caption{
    We display the memory consumption for each schemes, on different mesh sizes and timesteps.
}
  \label{fig:memory}
\end{figure}

\section{Related Work}

\paragraph{Graph networks for algorithmic reasoning and inductive biases}
Early work formalised graph networks as a vehicle for relational inductive bias \cite{Battaglia2018Relational}.  
Subsequent studies showed GNNs can learn to \emph{execute} classical algorithms such as shortest-path, topological sort, and dynamic programming \cite{Velickovic2019NEGA,Velickovic2021NAR,Dudzik2022DynamicProg}, with recent evidence of accurate out-of-distribution extrapolation \cite{Nerem2025ShortestPath}.  

\paragraph{Graph networks for physical simulation}
Interaction Networks \cite{Battaglia2016Interaction} and Neural Physics Engines \cite{Chang2017NPE} pioneered learned particle simulators; later work cast entire physical systems as graphs \cite{Sanchez2018PhysicsEngine,Sanchez2020Simulate}, marrying visual perception with GNN rollouts \cite{Mrowca2018Flexible}.  
Differentiable physics engines \cite{BelbutePeres2018DiffPhys} enabled gradient-based control, while MeshGraphNets \cite{Pfaff2021MeshGraph} and Neural Relational Inference \cite{Kipf2018NRI} extended the paradigm to irregular meshes and latent interaction discovery.

\paragraph{Physics-informed neural networks and graph PDE solvers}
PINNs introduced residual losses that embed PDEs into network training \cite{Raissi2019PINN}, later surveyed comprehensively in \cite{Karniadakis2021PINNReview}.  
Neural operators such as the Fourier Neural Operator \cite{Li2020FNO}, Multipole Graph Neural Operator \cite{Li2020MGNO}, and DeepONet \cite{Lu2021DeepONet} learn mesh-free mappings between infinite-dimensional function spaces, with a rigorous theory synthesised in \cite{Kovachki2023NeuralOperator}.  
Finite-volume-informed GNNs advance this line by minimising discrete conservation residuals \cite{Li2024FVGN,Li2025FVI,Salle2024Perf}.  
Complementary efforts learn data-driven discretisations \cite{BarSinai2019PNAS,Ehrhardt2023DeepFDM}. More importantly, PINNs are now used as tools to advance some of the most difficult physics problems such as the Navier-Stokes Millenial problem \cite{wang2023asymptoticselfsimilarblowupprofile,kumar2024investigatingabilitypinnssolve,wang2025highprecisionpinnsunbounded}. 

\paragraph{Symbolic and interpretable machine learning for physics discovery}
Symbolic regression has long sought human-readable laws from data \cite{Schmidt2009SR}, with AI Feynman \cite{Udrescu2020AIFeynman} blending neural approximations and heuristic searches.  
GNNs with sparsity-promoting priors can expose the latent equations learned during training \cite{Cranmer2020Symbolic}, a capability recently validated by re-deriving orbital mechanics from planetary trajectories \cite{Lemos2022Orbit}.  

\section{Conclusion}

In this work, we have demonstrated a deep and interpretable connection between Graph Neural Networks and the FVM for solving partial differential equations. Our contributions bridge the gap between deep learning for physical simulation and the discovery of numerical methods.
\paragraph{}
First, we established that a GNN, trained on a highly constrained dataset of simple two-node graphs, can learn the fundamental principles of a first-order scheme. More importantly, it can generalize this knowledge to solve the equation on large, unstructured, and entirely out-of-distribution meshes, a feat made possible by enforcing sparsity in the model's architecture.
\paragraph{}
Second, by applying symbolic regression to the trained network's components, we proved that the GNN does not merely approximate the solution but learns the exact analytical form of the conventional finite volume update rule. This result highlights the GNN's ability to internalize the underlying governing laws from raw data.
\paragraph{}
Our most significant contribution is a novel methodology for scientific discovery. By training a GNN with either a wider receptive field (2-hop neighborhoods) or 2 message passing layers, using only a physics-informed loss function, without any ground-truth simulation data, we enabled it to discover a new, higher-order finite volume scheme. The discovered schemes, which includes a non-trivial correction term based on second-order neighbors, were shown through convergence analysis to be second-order accurate and demonstrably more precise than the standard first-order method.
\paragraph{}
While the GNN architecture used was intentionally designed to explore 2-hop interactions, this work serves as a proof of concept. The methodology opens a new frontier where GNNs are not just solvers but tools for scientific discovery, capable of automatically proposing novel and more accurate numerical methods for complex physical systems. Future research should explore more general architectures to mitigate inductive biases, apply this technique to a broader range of complex PDEs, and further investigate the theoretical properties of the discovered schemes.

\section{Acknowledgements}

We would like to thank Pablo Jeken-Rico and Loic Chadoutaud for valuable discussions.

\noindent The authors acknowledge the financial support from ERC grant no 2021-CoG-101045042, CURE. Views and opinions expressed are however those of the author(s) only and do not necessarily reflect those of the European Union or the European Research Council. Neither the European Union nor the granting authority can be held responsible for them.

\newpage


\bibliography{main} 
\bibliographystyle{plain}

\newpage
\appendix

\tableofcontents

\section{Definitions and Datasets}

\subsection{Definitions}

\begin{definition}
  A First-Order Finite Volume Scheme (SFVS) in a subdomain $\Omega \subset \mathbb{R}^2$, partitioned into disjoint cells $\Omega_P$, applied to the Heat Equation, is defined for each cell, for each timestep, by the following:

  \begin{equation}
    T_P^{n+1} = T_P^n + \frac{\Delta t}{V_P} \sum_{N \in \mcN(\Omega_P)} \alpha \frac{A_{f_{PN}}}{\delta_{PN}} (T_N^n - T_P^n) + \Delta 
t S_P
    \label{eq:default-fvm-annexe}
  \end{equation}
\end{definition}

\begin{definition}[MPS]
  An L-layer Message Passing GNN (MPS) $\mathcal{M}_{\theta}$, with $m$ hidden layers of dimensions $d$, with a set of weights $\theta$, computes for all $l \in \llbracket 1, L\rrbracket$

\begin{equation}
    h_r^l = f^{\text{up}} \Big(
    h_r^{l-1}, \displaystyle\sum_{e \in \mcN(\Omega_r)} f^{\text{agg}}(\mathbf{e}_k,\mathbf{v}_{r_k},\mathbf{v}_{s_k})
    \Big)
    \label{eq:mps_appendix}
\end{equation}

  where $h_r^0 = \mathbf{v}_r$,  $f^{\text{up}}$ and $f^{\text{agg}}$ are $m$-layers MLP with ReLU activation functions except on the last layer.
with weights $\theta$.
\end{definition}

\begin{lemma}[Fourier stability condition]
\label{lemma:cfl}
For the first-order scheme (\autoref{eq:default-fvm}) to be numerically stable, the timestep $\Delta t$ must satisfy the Fourier condition:
  \begin{equation}
    \Delta t \leq \min_{P} \frac{V_P}{\displaystyle\sum_{N \in \mcN(\Omega_P)} \alpha \frac{A_{f_{PN}}}{\delta_{PN}}}
  \end{equation}
\end{lemma}

\begin{remark}
    In particular, this means that for every node $P$, we have:
\begin{equation}
    \Delta t \leq\frac{V_P}{\displaystyle\sum_{N \in \mcN(\Omega_P)} \alpha \frac{A_{f_{PN}}}{\delta_{PN}}}
  \end{equation}
and since temperatures are normalized, we have 
$$\displaystyle\sum_{N \in \mcN(\Omega_P)} \alpha \frac{A_{f_{PN}}}{\delta_{PN}} (T_N^n - T_P^n) \leq 2\displaystyle\sum_{N \in \mcN(\Omega_P)} \alpha \frac{A_{f_{PN}}}{\delta_{PN}}$$
and thus:
\begin{equation}
    \displaystyle\sum_{N \in \mcN(\Omega_P)} \alpha \frac{A_{f_{PN}}}{\delta_{PN}} (T_N^n - T_P^n) \leq 2\frac{V_P}{\Delta t}
  \end{equation}
\end{remark}

\section{Single Layer GNN implements FVM}
\subsection{Toy Example}
We start by working on the toy example presented in the beginning of \autoref{sec:toy}.
Let's denote the inter-cells component as $\bij=\alpha \frac{A_{ij}}{\delta_{ij}} (T_j - T_i)$.
Let $\mcG$ be a graph from the training set $\mathcal{D}_{\text{train}}$.
We have one step of FVM defined as:

\begin{equation}
  T_i^{n+1} = T_i^n + \frac{\Delta t}{V_P} \bij + \Delta t S_i
  \label{eq:training-fvm}
\end{equation}

Let $\mathcal{M}_{\theta}$ be a 1-layer GNN, with no hidden layers.
We thus have:

\begin{equation}
  h_i^{(1)} = W_u \left[ T_i, \frac{\Delta t}{V_i} \sum_{j \in \mcN(i)} (w_{\text{agg}} \bij + b_{\text{agg}}), \Delta t S_i \right] + b_u
\end{equation}

where $W_u \in \R^{3}$.
For a graph $\mcG$ from $\mathcal{D}_{\text{train}}$, this gives us:

\begin{equation}
  h_i^{(1)} = W_u \left[ T_i, \frac{\Delta t}{V_i}(w_{\text{agg}} \bij + b_{\text{agg}}), \Delta t S_i \right] + b_u
\end{equation}

\begin{lemma}
  \label{lemma:simpleweights}
  If for all inputs $\mcG$ from $\mathcal{D}_{\text{train}}$, we have $\mathcal{M}_{\theta}(\mcG) = \mcF(\mcG)$, then:
  \[
    \begin{cases}
      W_u^1 = 1                 \\
      W_u^3 = 1                 \\
     
 W_u^2 w_{\text{agg}} = 1  \\
      \frac{\Delta t}{V_i} W_u^2b_{\text{agg}}+ b_u = 0
    \end{cases}
  \]
\end{lemma}

\begin{proof}

  Let's start by taking a node from the graph $G(1,1,0,0)$ (which leads to $\bij = 0$).
This gives us $\mcF(\mcG)_i = T_i$ and:

  \begin{equation}
    h_i^{(1)} = W_u \left[ T_i, \frac{\Delta t}{V_i}b_{\text{agg}}, 0 \right] + b_u = W_u^1T_i + \frac{\Delta t}{V_i} W_u^2b_{\text{agg}}+ b_u
  \end{equation}

  Simplifying the system with $G(0,0,0,0)$, and similarly with $G(0,0,1,1)$, we get the following:

  \[
    \begin{cases}
      W_u^1 + \frac{\Delta t}{V_i} W_u^2b_{\text{agg}}+ b_u = 1 \\
      \frac{\Delta t}{V_i} W_u^2b_{\text{agg}}+ b_u = 0              \\
      W_u^3 + \frac{\Delta t}{V_i} 
W_u^2b_{\text{agg}}+ b_u = 1 \\
    \end{cases}
  \]

  which simplifies into:

  \[
    \begin{cases}
      W_u^1 = 1                                    \\
      \frac{\Delta t}{V_i} W_u^2b_{\text{agg}}+ b_u = 0 \\
      W_u^3 = 1                
                    \\
    \end{cases}
  \]

Let's now sample a random graph from the training set, we have:

  \begin{equation}
    \lnorm*{h_i^{(1)} - \mcF(\mcG)_i}= \lnorm*{\frac{\Delta t}{V_i}  \bij (W_u^2w_{\text{agg}} - 1)}= 0
  \end{equation}

  since $W_u^1 = W_u^3 = 1$ and $\frac{\Delta t}{V_i} W_u^2b_{\text{agg}}+ b_u = 0$.
This gives us $W_u^2w_{\text{agg}}=1$ (and thus $W_u^2 \neq 0$). 
\end{proof}

\begin{remark}
  Adding $L_1$ regularisation adds a last constraint on the system: minimizing the absolute value of $W_u^2w_{\text{agg}}=1$.
The minimum of $x \mapsto \lnorm*{ x} + \frac{1}{\lnorm*{ x }}$ is obtained on $x=1$.
Adding $L_1$ regularisation thus leads to the additional equalities $W_u^2 = w_{\text{agg}}=1$.

It also leads to $b_{\text{agg}} = b_u = 0$. The system thus becomes:
\[
    \begin{cases}
      W_u^1 = W_u^3 = W_u^2 = w_{\text{agg}} = 1                 \\
b_{\text{agg}}= b_u = 0
    \end{cases}
  \]
\end{remark}

We now demonstrate the main results of \autoref{sec:toy} that we restate here:
\subsection{Out-of-distribution generalisation by simple GNN}
\begin{theorem}[Out-of-distribution generalisation by simple GNN]
  Let $\varepsilon > 0$. Let $\mathcal{M}_{\theta}$ be a 1-layer GNN where $\forall \mcG \in \mathcal{D}_{\text{train}}, \forall r \in \mathcal{V}({\mcG})$: $\lnorm*{ \mathcal{M}_{\theta}(\mcG)_r - \mcF(\mcG)_r} < \varepsilon$. Let $\mcN$ be the maximum number of neighbors in $\mathcal{D}_{\text{train}}$.
Then, for any graph $\mcG$ and any node $r \in \mathcal{V}({\mcG})$:

  \begin{equation}
    \lnorm*{ \mathcal{M}_{\theta}(\mcG)_r - \mcF(\mcG)_r} <  \varepsilon \left(4 + \frac{4}{\Delta t} + \mcN\right)
  \end{equation}

\end{theorem}

\paragraph{Proof sketch}

\begin{enumerate}
  \item We start by explicitely writing out a GNN update using the different features and trainables weights in the network.
  \item We apply our main hypothesis on the training set on specific graphs, such as one with a unit temperature in one cell, and all source terms and the other cell temperature set at 0.
  \item This gives us several inequalities on the networks weights. Finally, we use those inequalities to bound the error on a general graph. 
\end{enumerate}

\begin{proof}
\label{proof:th1}
  We expand $h_i^{(1)}$ and simplify it for a 2-cells graph from $\mathcal{D}_{\text{train}}$:
  \begin{align}
    h_i^{(1)} & = T_i W_{1}^u + \Delta t S_i W_{3}^u + \frac{\Delta t}{V_i} W_{2}^u \sum_j (w_{\text{agg}} \bij + b_{\text{agg}}) + b_u
    \\ &= T_i W_{1}^u + \Delta t S_i W_{3}^u + \frac{\Delta t}{V_i} W_{2}^u (w_{\text{agg}} \bij + b_{\text{agg}}) + b_u
    \\ &= T_i W_{1}^u + \Delta t S_i W_{3}^u + \frac{\Delta t}{V_i} W_{2}^u w_{\text{agg}} \bij + \frac{\Delta t}{V_i} W_{2}^ub_{\text{agg}} + b_u
  \end{align}

  We substract this from the ground truth $\mcF$ and get:

  \begin{align}
      & \lnorm*{ \mcF(\mcG)_i-h_i^{(1)}}
    \\ &=
    \lnorm*{ T_i + \Delta t S_i + \frac{\Delta t}{V_i} \bij - T_i W_{1}^u - \Delta t S_i W_{3}^u - \frac{\Delta t}{V_i} W_{2}^u w_{\text{agg}} \bij - \frac{\Delta t}{V_i} W_{2}^ub_{\text{agg}} - b_u}
    \\ &= \lnorm*{ T_i (1-W_{1}^u) + \Delta t S_i (1-W_{3}^u) + \frac{\Delta t}{V_i} \bij (1 - W_{2}^u w_{\text{agg}}) + \frac{\Delta t}{V_i} W_{2}^ub_{\text{agg}} + b_u }
    \\ &\leq \varepsilon
  \end{align}

  Similar to the strategy used previously, we simplify this inequality for $G(1,1,0,0)$ and $G(0,0,1,1)$. This leads to the following 2 inequalities:

  \[
    \begin{cases}
      \lnorm*{(1-W_{1}^u) + \frac{\Delta t}{V_i} W_{2}^ub_{\text{agg}} + b_u } \leq \varepsilon
      \\
      \lnorm*{ \Delta t(1-W_{3}^u) + \frac{\Delta t}{V_i} W_{2}^ub_{\text{agg}} + b_u } \leq \varepsilon
    \end{cases}
  \]

  We now apply the inequality to $G(0,0,0,0)$:
  
\begin{align}
\lnorm*{ \frac{\Delta t}{V_i} W_{2}^ub_{\text{agg}} + b_u }
     &\leq \varepsilon
\end{align}

This gives us:

\begin{align}
      & \lnorm*{ (1-W_{1}^u)}
    \\ &=
    \lnorm*{(1-W_{1}^u) + \frac{\Delta t}{V_i} W_{2}^ub_{\text{agg}} + b_u - \left( \frac{\Delta t}{V_i} W_{2}^ub_{\text{agg}} + b_u \right) } & \text{(by adding and substracting term)}
    \\ &\leq \lnorm*{(1-W_{1}^u) + \frac{\Delta t}{V_i} W_{2}^ub_{\text{agg}} + b_u} + \lnorm*{ \frac{\Delta t}{V_i} W_{2}^ub_{\text{agg}} + b_u } & \text{(by the triangle inequality)}
    \\ &\leq 2\varepsilon
\end{align}

and similarly $\lnorm*{ (1-W_{3}^u)} \leq \frac{2\varepsilon}{\Delta t}$

Finally, applying the inequality to $G(0,1,0,0)$: 

\begin{align}
&\lnorm*{ \frac{\Delta t}{V_i} \alpha \frac{A_{ij}}{\delta _{ij}} (1 - W_{2}^u w_{\text{agg}}) + \frac{\Delta t}{V_i} W_{2}^ub_{\text{agg}} + b_u }
    \\ &\leq \varepsilon
  \end{align}

Since $\frac{\Delta t}{V_i} \alpha \frac{A_{ij}}{\delta _{ij}}$ simplifies into $\Delta t$ and using the same strategy as above, we have $\lnorm*{ (1-W_{2}^u w_{\text{agg}})} \leq \frac{2\varepsilon}{\Delta t}$

  Let's now handle the proof of the main result. Let $\mcG$ be a random graph not from the training set. We have:

  \begin{align}
      & \lnorm*{ \mcF(\mcG)_i-h_i^{(1)}}
    \\ &=
    \lnorm*{ T_i(1-W_{1}^u) + \Delta t S_i(1-W_{3}^u) + (1-W_{2}^u w_{\text{agg}})\frac{\Delta t}{V_i}\sum_{j \in \mcN(i)}\bij  - \abs*{\mcN(i)}\frac{\Delta t}{V_i} W_{2}^ub_{\text{agg}} - b_u}
    \\ &\leq  \lnorm*{ T_i(1-W_{1}^u)} + \lnorm*{ \Delta t S_i(1-W_{3}^u) } + \lnorm*{ (1-W_{2}^u w_{\text{agg}})\frac{\Delta t}{V_i}\sum_{j \in \mcN(i)}\bij } + \lnorm*{\abs*{\mcN(i)}\frac{\Delta t}{V_i} W_{2}^ub_{\text{agg}} + b_u} 
    \\ &\leq  \lnorm*{ T_i}2\varepsilon + \lnorm*{ \Delta t S_i} \varepsilon\frac{2}{\Delta t} + 2\varepsilon\frac{2}{\Delta t}\lnorm*{ \frac{\Delta t}{V_i}\sum_{j \in \mcN(i)}\bij } + \varepsilon\abs*{\mcN(i)}
    \\ &\leq  2\varepsilon + 2\varepsilon + \varepsilon\frac{4}{\Delta t} + \varepsilon\mcN
  \end{align}

where we use the normalization of $T_i$ and $S_i$ for the first two terms, and apply the Fourier conditions \autoref{lemma:cfl} to the third one.

\end{proof}

\section{Theoretical Result: Expressivity of GNNs for Finite Volume Schemes}
\label{appendix:main-demo}

We tackle here the more general case of a wide and deep GNN. More precisely, we work with the same training dataset, but we now let $\mathcal{M}_{\theta}$ be an $L$-layers GNN, with $m$-layers ReLU MLPs of width $d$. Importantly, both in $f^{\text{up}}$ and $f^{\text{agg}}$, the last layer does not have a ReLU activation function.

\paragraph{Proof sketch}

\begin{enumerate}
  \item We start by showing that we can implement a perfect first order scheme using $2mL+2m+4$ parameters (Lemma \ref{lemma:f-perf-solution}).
  \item We then show that any solution with stricly less than $2mL+2m+4$ parameters will achieve a higher loss (Lemma \ref{lemma:minimal-parameters}). This leads to 
    \autoref{theorem:optimality} where we demonstrate that a small training error requires at least $2mL+2m+4$ parameters.
  \item Finally, we apply the previous results to vastly simplify the weights in the network, leading to a final proof very similar to the previous proof \ref{proof:th1}.
\end{enumerate}

\subsection{Sparsity consideration}

\begin{lemma}
  \label{lemma:f-perf-solution}
  An $L$-layers GNN, with $m$-layers ReLU MLPs of width $d$ can reproduce perfectly $\mcF$ on any graph $\mcG$ with exactly $2m(L+1)+4$ non-zero parameters.
\end{lemma}

\begin{proof}
  Since $\mcF$ aggregates information only in a 1-hop neighborhood, it's easy to see that we need the first layer to implement $\mcF$, and the remaining $L-1$ to implement an identity function node wise.

  This means that for each on those $L-1$ layers, we can set-up:
  \begin{itemize}
    \item $f^{\text{agg}}_1 : \R \rightarrow \R^d$, $f^{\text{agg}}_2, ...,f^{\text{agg}}_{m-1} : \R^d \rightarrow \R^d$ and $f^{\text{agg}}_m : \R^d \rightarrow \R$ to have all parameters at 0
    \item $f^{\text{up}}_1 : \R^3 \rightarrow \R^d$, $f^{\text{up}}_2, ...,f^{\text{up}}_{m-1} : \R^d \rightarrow \R^d$ and $f^{\text{up}}_m : \R^d \rightarrow \R$ to have all parameters at 0 except 2 per layers at -1 and 1.
  \end{itemize}

The key here is to notice that if an input is negative, then it would be canceled out by the ReLU activation function. To that end, each weights matrix needs at least a $+1$ and a $-1$ coefficient, to make sure the information gets propagated properly. For example for an input $x \in \R$: 

  \begin{align}
    \sigma(Wx) & = [x, 0, ...,0] \text{ if $x > 0$}
    \\ &=  [0, \lnorm*{x}, ...,0] \text{ if $x < 0$}
  \end{align}

  This makes a total of $2m(L-1)$ non-zero parameters.

  We will now build the aggregation network to replicate exactly $\frac{A_{ij}}{\delta _{ij}}(T_j - T_i)$. We set all the biases to 0. The key thing here is to notice that $(T_j - T_i)$ can be negative as well and needs 2 non-null weights with opposite signs to produce:

  \begin{align}
    \sigma(W^a_1 \bij) & = [\bij, 0, ...,0] \text{ if $(T_j - T_i) > 0$}
    \\ &=  [0, \lnorm*{\bij}, ...,0] \text{ if $(T_j - T_i) < 0$}
  \end{align}

  The remaining $m-2$ layers (all hidden layers) can then implement identity weights $W^a_p \in \R^{d\times d}$ restricted to the first 2 elements of the diagonal, similar to above.

  \[
    W^a_p = \begin{bmatrix}
      1 & 0 & 0 & \dots \\
      0 & 1 & 0 & \dots \\
      0 & 0 & 0 & \dots \\
      \vdots & \vdots & \vdots & \ddots &   \\
    \end{bmatrix}
  \]

  The final set of weights finally has to replace the original sign of $\bij$ using $W^a_m \in \R^{d} = [1,-1,0,....,0]$, similar to above. This implements an identity function for inputs in $\R$ with exactly $2m$ parameters.

  We will now build the update network to replicate exactly $T_i + \frac{\Delta t}{V_i} \sum_{j \in \mcN(i)} \alpha \frac{A_{ij}}{\delta_{ij}} (T_j - T_i) + \Delta t S_i$. Given our inputs and the sum aggregation function, it's easy to see that if our network implement an identity function for each input, we replicate exactly $\mcF$. Similarly to the aggregation network, the term $\sum_{j \in \mcN(i)} \alpha \frac{A_{ij}}{\delta_{ij}} (T_j - T_i)$ can still be negative and thus be canceled out by $\sigma$. This statement holds as well for $T_i$ and $\Delta t S_i$.

  Let's set $W^u_1 \in \R^{3\times d}$ to:

  \[
    W^u_1 = \begin{bmatrix}
      1 & 1 & 1 & 0 & \dots \\
      -1 & -1 & -1 & 0 &\dots \\
      \vdots & \vdots & \vdots & \vdots & \ddots &   \\
    \end{bmatrix}
  \]

  This gives us:

  \begin{align}
    \sigma(W^u_1 x) & = \begin{bmatrix}
      T_i+\frac{\Delta t}{V_i} \sum_{j \in \mcN(i)} \bij+\Delta t S_i & 0 & \dots \\
      0 & 0 &\dots \\
      \vdots & \vdots & \ddots &   \\
    \end{bmatrix}
  \end{align}

  if the sum is positive, and 

  \begin{align}
    \sigma(W^u_1 x) & = \begin{bmatrix}
      0 & 0 & \dots \\
      \lnorm*{T_i+\frac{\Delta t}{V_i} \sum_{j \in \mcN(i)} \bij+\Delta t S_i} & 0 &\dots \\
      \vdots & \vdots & \ddots &   \\
    \end{bmatrix}
  \end{align}
    otherwise, where $x = [T_i, \frac{\Delta t}{V_i} \sum_{j \in \mcN(i)} \bij , \Delta t S_i]$. The remaining layers simply have to implement an identity function, using 2 non-zero parameters per layer. This leads to a total of $3+(m-1)$ non-zero parameters.
\end{proof}

\begin{remark}
  This is exactly the number of parameters we find with $L=1$ and $m=2$ in \autoref{sec:toy}.
\end{remark}

\begin{remark}
  It is important to note that this solution leads to a loss

  \begin{align}
    \mcL(\mathcal{M}_{\theta}, \mathcal{D}_{\text{train}}) & = \mcL_{\text{MAE}}(\mathcal{M}_{\theta}, \mathcal{D}_{\text{train}}) + \vert\vert \theta \vert\vert_0
    \\ &= 0 + \eta \left(2m(L+1)+4 \right)
  \end{align}

  This means that for any other solution that reproduces perfectly $\mcF$ with $p$ extra parameters, the loss will be at least $p\eta$ higher.
\end{remark}

\begin{lemma}
  \label{lemma:minimal-parameters}
  An $L$-layers GNN, with $m$-layers ReLU MLPs of width $d$ with less than $2m(L+1)+4$ non-zero parameters cannot have a non-zero error on the training set $\mathcal{D}_{\text{train}}$.
\end{lemma}

\begin{proof}

  The model must have at least one propagation from edge to nodes, and then no loss of information. This ensures that the $L-1$ last layers do not lose information, and have at least each 2 non-zero parameters per hidden layers. This leads to $2m(L-1)$ non-zero parameters. We now focus on the first layer. Similarly, we need to propagate some of the edge information one way or the other to each node, leading to at least $m$ non-zero parameters for the aggregation network of the first layer. Finally, we need to at least mix the node inputs together and propagate them without loss of information due to signs, leading to $2(3 + (m-1))$ parameters.

  In total, this gives us $2mL + m+4$ parameters. However, having any $\bij < 0$ will lead to a propagation of zeros and thus a loss of information. Let's assume we have a set of weights $\theta$ with $\vert\vert\theta \vert\vert_0 < 2m(L+1)+4$. Since we already have  $\vert\vert \theta \vert\vert_0 \geq 2mL + 2$, it means we can implement 2 identity functions, with ReLU activation functions, for negative inputs, with at most $(2m(L+1)+4) - (2mL + m+4) -1 = m-1$ parameters. Since we have $m$ layers, this concludes the proof.

\end{proof}

\subsection{Optimality of the Sparse SFVS Implementation}
\begin{theorem}[Optimality of the Sparse SFVS Implementation]
\label{theorem:optimality}
Let $\mathcal{M}_{\theta}$ be an $L$-layer Message Passing GNN (MPS) with $m$-layer MLPs ($f^{\text{up}}$, $f^{\text{agg}}$) using ReLU activations except on the final layer. Let the loss function be $\mcL(\mathcal{M}{\theta}, \mathcal{D}{\text{train}}) = \mcL_{\text{MAE}}(\mathcal{M}{\theta}, \mathcal{D}{\text{train}}) + \eta\vert\vert \theta \vert\vert_0$. Let $\varepsilon > 0$ be sufficiently small, specifically $\varepsilon < \eta$. Let $N_{sparse} = 2m(L+1)+4$.

Define $C(K)$ as the minimum MAE achievable on $\mathcal{D}_{\text{train}}$ by any configuration $\theta$ with $||\theta||_0 = K$. We assume the regularization parameter $\eta$ is sufficiently small such that:
  \[ \eta < \min_{K < N_{sparse}} \frac{C(K)}{N_{sparse}-K} \]

If a set of parameters $\theta$ achieves a loss $\mcL(\mathcal{M}_{\theta}, \mathcal{D}_{\text{train}}) \leq \eta N_{sparse} + \varepsilon$, then $\mathcal{M}_{\theta}$ must have exactly $N_{sparse} = 2m(L+1)+4$ non-zero parameters, structured to approximate the SFVS function $\mcF$. Furthermore, $\eta N_{sparse}$ is the global minimum of the loss function $\mcL$.
\end{theorem}

\begin{proof}
First, the existence of an implementation achieving perfect accuracy ($\mcL_{\text{MAE}}=0$) with exactly $N_{sparse} = 2m(L+1)+4$ non-zero parameters is established by \autoref{lemma:f-perf-solution}. This implementation achieves a total loss $\mcL = 0 + \eta ||\theta||_0 = \eta N_{sparse}$.

Consider any set of parameters $\theta'$ such that $\mcL(\mathcal{M}_{\theta'}, \mathcal{D}_{\text{train}}) \leq \eta N_{sparse} + \varepsilon$. The loss is given by $\mcL(\mathcal{M}_{\theta'}, \mathcal{D}_{\text{train}}) = \mcL_{\text{MAE}}(\mathcal{M}_{\theta'}, \mathcal{D}_{\text{train}}) + \eta ||\theta'||_0$. Since $\mcL_{\text{MAE}} \geq 0$, we must have $\eta ||\theta'||_0 \leq \mcL(\mathcal{M}_{\theta'}, \mathcal{D}_{\text{train}}) \leq \eta N_{sparse} + \varepsilon$. If $\theta'$ had $N' > N_{sparse}$ non-zero parameters, then $||\theta'||_0 = N' \geq N_{sparse} + 1$. In this case, the loss would be $\mcL(\mathcal{M}_{\theta'}, \mathcal{D}_{\text{train}}) \geq \eta ||\theta'||_0 \geq \eta (N_{sparse} + 1) = \eta N_{sparse} + \eta$. Since we assumed $\varepsilon < \eta$, this contradicts the condition $\mcL(\mathcal{M}_{\theta'}, \mathcal{D}_{\text{train}}) \leq \eta N_{sparse} + \varepsilon$. Therefore, any parameter set $\theta'$ achieving a loss less than or equal to $\eta N_{sparse} + \varepsilon$ must satisfy $||\theta'||_0 \leq N_{sparse}$. This establishes $N_{sparse}$ as an upper bound on the number of non-zero parameters for near-optimal solutions.

Now, we derive a lower bound on the required sparsity by analyzing the structural requirements imposed by the SFVS function $\mcF$ and the training data $\mathcal{D}_{\text{train}}$. Using \autoref{lemma:minimal-parameters},  we show that the total minimum number of non-zero parameters required is $2m(L+1)+4$. This is exactly $N_{sparse}$.

Since we established both an upper bound ($||\theta'||_0 \leq N_{sparse}$) and a lower bound ($||\theta'||_0 \geq N_{sparse}$) for any parameter set $\theta'$ achieving the near-optimal loss $\mcL(\mathcal{M}_{\theta'}, \mathcal{D}_{\text{train}}) \leq \eta N_{sparse} + \varepsilon$, we conclude that such a solution must have exactly $||\theta'||_0 = N_{sparse}$ non-zero parameters.

Furthermore, we establish that $\eta N_{sparse}$ is the global minimum. The implementation from \autoref{lemma:f-perf-solution} achieves $\mcL_{\text{MAE}}=0$ with $N_{sparse}$ parameters, yielding loss $\eta N_{sparse}$. Any other implementation $\theta'$ has loss $\mcL(\theta') = \mcL_{\text{MAE}}(\theta') + \eta ||\theta'||_0$.
\begin{itemize}
    \item If $||\theta'||_0 = K > N_{sparse}$, the loss is $\mcL(\theta') \geq 0 + \eta (N_{sparse}+1) > \eta N_{sparse}$.
    \item If $||\theta'||_0 = N_{sparse}$, the loss is $\mcL(\theta') = \mcL_{\text{MAE}}(\theta') + \eta N_{sparse} \geq \eta N_{sparse}$.
    \item If $||\theta'||_0 = K < N_{sparse}$, by \autoref{lemma:minimal-parameters}, $\mcL_{\text{MAE}}(\theta') > 0$. The loss is $\mcL(\theta') \geq C(K) + \eta K$.
\end{itemize}

By the assumption on $\eta$, we have $\eta < \frac{C(K)}{N_{sparse}-K}$, which rearranges to $C(K) > \eta (N_{sparse}-K)$. Therefore, the loss satisfies: \[ \mcL(\theta') \geq C(K) + \eta K > \eta (N_{sparse}-K) + \eta K = \eta N_{sparse} \]

Thus, the minimum achievable loss value is indeed $\eta N_{sparse}$, attained by the optimally sparse implementation described in \autoref{lemma:f-perf-solution}.
\end{proof}

\begin{theorem}[Out-of-Distribution Generalization Bound]
  \label{theorem:generalization_bound}
  Let $\mathcal{M}_{\theta}$ be an $L$-layer Message Passing GNN (MPS) with $m$-layer MLPs ($f^{\text{up}}$, $f^{\text{agg}}$) using ReLU activations except on the final layer. Let the loss function be $\mcL(\mathcal{M}{\theta}, \mathcal{D}{\text{train}}) = \mcL_{\text{MAE}}(\mathcal{M}{\theta}, \mathcal{D}{\text{train}}) + \eta\vert\vert \theta \vert\vert_0$. Let $\varepsilon > 0$ be sufficiently small, specifically $\varepsilon < \eta$. Let $N_{sparse} = 2m(L+1)+4$.

  For any graph $\mcG$, we have:

  \begin{equation}
    \label{eq:main-theorem-2}
    \lnorm*{ \mathcal{M}_{\theta}(\mcG)_r - \mcF(\mcG)_r} < 2\varepsilon \left( 1 + \frac{2}{\Delta t}\right)
  \end{equation}
\end{theorem}

\begin{proof}
  From \autoref{theorem:optimality}, we know that any GNN achieving a near-optimal loss on our training set must adopt a specific sparse structure with $N_{sparse} = 2m(L+1)+4$ non-zero parameters. We can simplify the structure of the neural network further, having all biases set to 0, and the minimal amount of non-zero weights parameters.
This leads to a formulation of $h^1_i$ as mostly a product of weights. More precisely, we have, using the following notation: $ \mathbf{U}^1   = \displaystyle\prod_{l=2}^{m} u^1_l$, $\mathbf{U}^2   = \displaystyle\prod_{l=2}^{m} u^2_l$, $\mathbf{A}^1 = \displaystyle\prod_{l=1}^{m} a^1_l$ and $\mathbf{A}^2 = \displaystyle\prod_{l=1}^{m} a^2_l$:

  \begin{equation}
    h^1_i = \mathbf{U}^1x \text{ if $x \geq 0$ and } \mathbf{U}^2x \text{ otherwise }
  \end{equation}

  where $x =  u_1^1T_i +  u_1^2 \frac{\Delta t}{V_i}\displaystyle\sum _{\substack{j \in \mcN(i) \\ \bij > 0}} \mathbf{A}^1\bij + u_1^2 \frac{\Delta t}{V_i}\displaystyle\sum _{\substack{j \in \mcN(i) \\ \bij < 0}} \mathbf{A}^2\bij + u_1^3\Delta t S_i$

  While the use of products of simple weights is clear from the beginning of the demonstration, the removal of the ReLU activation function needs more work. For the remaining of the proof, we omit the distinction between $\mathbf{U}^1$ and $\mathbf{U}^2$ and simply write $\mathbf{U}$\footnote{thus making the assumption that $x$ is always positive.}. By symmetry, the same inequalities will hold for $x < 0$ which will conclude the proof for any $x \in \R$.
  
  We write $h^1_i$ for $G(0,1,0,0)$ (where $\bij > 0$) and for $G(1,0,0,0)$ (where $\bij < 0$):

  \begin{align}
    \lnorm*{ h^1_i - \bij }       & = \lnorm*{
    \mathbf{U} \left( u_1^2 \frac{\Delta t}{V_i}\left( a_m^1\sigma\left(...\sigma(a^1_1\bij)\right) + a_m^2\sigma\left(...\sigma(a^2_1\bij)\right) \right) \right)
    - \bij
    } \leq \varepsilon \quad \\&\text{where $\bij = \frac{A_{ij}}{\delta _{ij}}$ } \\
    \lnorm*{ h^1_i - \bij - T_i } & = \lnorm*{
    \mathbf{U}u_1^1T_i +
    \mathbf{U} \left( u_1^2 \frac{\Delta t}{V_i}\left( a_m^1\sigma\left(...\sigma(a^1_1\bij)\right) + a_m^2\sigma\left(...\sigma(a^2_1\bij)\right) \right) \right)
    - \bij - T_i
    } \leq \varepsilon \quad \\&\text{where $\bij = - \frac{A_{ij}}{\delta _{ij}}$ }
  \end{align}

  Since after the first ReLU passage the output is positive, and we need both cases to hold, we have $a^1_1 a^2_1 < 0$. Let's assume that $a^2_1 < 0$. Similarly, since after the first ReLU, all numbers will be positive and cannot be zeroed out, we must have $\forall l \in \llbracket 2, m-1\rrbracket : a^2_l > 0$ and $\forall l \in \llbracket 1, m\rrbracket : a^1_l > 0$. This leads to the following simplification of the 2 equations above:

  \begin{align}
    \lnorm*{ h^1_i - \bij }       & = \lnorm*{
    \frac{\Delta t}{V_i}\bij (\mathbf{U}u_1^2\mathbf{A}^1-1)
    } \leq \varepsilon \quad \text{where $\bij = \frac{A_{ij}}{\delta _{ij}}$ } \\
    \lnorm*{ h^1_i - \bij - T_i } & = \lnorm*{
    \mathbf{U}u_1 -1 +
    \frac{\Delta t}{V_i}\bij(\mathbf{U}u_1^2\mathbf{A}^2-1)
    } \leq \varepsilon \quad \text{where $\bij = - \frac{A_{ij}}{\delta _{ij}}$ }
  \end{align}

  Notice that we have $\frac{A_{ij}}{V_i\delta _{ij}} =1$ in the training dataset, leading to:

  \begin{align}
    \lnorm*{
    \Delta t(\mathbf{U}u_1^2\mathbf{A}^1-1)
    } \leq \varepsilon  \\
    \lnorm*{
    \mathbf{U}u_1 -1 +
    \Delta t(\mathbf{U}u_1^2\mathbf{A}^2-1)
    } \leq \varepsilon 
  \end{align}

   We can now write $h^1_i$ for $G(1,1,0,0)$ and $G(0,0,1,1)$:

  \begin{align}
    \lnorm*{ h^1_i - T_i }          & = \lnorm*{
    \mathbf{U}u_1^1 - 1
    } \leq \varepsilon \\
    \lnorm*{ h^1_i - \Delta t S_i } & = \lnorm*{
    \mathbf{U}u_3^1\Delta _t - \Delta _t
    } \leq \varepsilon
  \end{align}

  Finally, using the same tricks as before, we get: 
  
  \begin{align}
    & \lnorm*{
    \mathbf{U}u_1^2\mathbf{A}^2\frac{\Delta t}{V_i}\bij
    - \frac{\Delta t}{V_i}\bij
    } \\ &= \lnorm*{ \mathbf{U}u_1^1T_i +
    \mathbf{U}u_1^2\mathbf{A}^2\frac{\Delta t}{V_i}\bij
    - \frac{\Delta t}{V_i}\bij - T_i - \mathbf{U}u_1^1T_i + T_i
    }
    \\ &\leq
    \lnorm*{
    \mathbf{U}u_1^1T_i - T_i
    } + \lnorm*{
    \mathbf{U}u_1^1T_i +
    \mathbf{U}u_1^2\mathbf{A}^2\frac{\Delta t}{V_i}\bij
    - \frac{\Delta t}{V_i}\bij - T_i
    } \\ &\leq 2\varepsilon
  \end{align}

  Let's now sample a graph $\mcG$ outside of the distribution data and compute the difference:

  \begin{align}
    & \lnorm*{ \mathcal{M}_{\theta}(\mcG)_r - \mcF(\mcG)_r} \\ &=
    \lnorm*{
    T_i(1 - \mathbf{U}u_1^1) + \Delta t S_i (1- \mathbf{U}u_1^3) + \frac{\Delta t}{V_i} \displaystyle\sum _{\substack{j \in \mcN(i) \\ \bij > 0}} \bij\left( 1 - \mathbf{U}u_1^2\mathbf{A}^1 \right)
    +
    \frac{\Delta t}{V_i} \displaystyle\sum _{\substack{j \in \mcN(i) \\ \bij < 0}} \bij\left( 1 - \mathbf{U}u_1^2\mathbf{A}^2 \right)
    }   \\ &\leq
    \varepsilon + \varepsilon + 
    \frac{\Delta t}{V_i} \lnorm*{1 - \mathbf{U}u_1^2\mathbf{A}^1} \lnorm*{\displaystyle\sum _{\substack{j \in \mcN(i) \\ \bij > 0}} \bij} \\& \quad \quad \quad +
    \frac{\Delta t}{V_i} \lnorm*{1 - \mathbf{U}u_1^2\mathbf{A}^2} \lnorm*{\displaystyle\sum _{\substack{j \in \mcN(i) \\ \bij < 0}} \bij} \quad  \text{(by the triangle inequality and previous inequalities)}
    \\ &\leq
    2\varepsilon + 
    \frac{2\varepsilon}{V_i} \lnorm*{\displaystyle\sum _{\substack{j \in \mcN(i) \\ \bij > 0}} \bij} +
    \frac{2\varepsilon}{V_i} \lnorm*{\displaystyle\sum _{\substack{j \in \mcN(i) \\ \bij < 0}} \bij}  \quad  \text{(by $\lnorm*{
    \Delta t(\mathbf{U}u_1^2\mathbf{A}^k-1)
    } \leq \varepsilon$)}
    \\ &\leq
    2\varepsilon + 
    \frac{2\varepsilon}{V_i} \left( 
    \displaystyle\sum _{\substack{j \in \mcN(i) \\ \bij > 0}} \lnorm*{\bij}
    + \displaystyle\sum _{\substack{j \in \mcN(i) \\ \bij < 0}} \lnorm*{\bij}
    \right)  \quad  \text{(by the triangle inequality)}
    \\ &\leq
    2\varepsilon + 
    \frac{2\varepsilon}{V_i}
    \displaystyle\sum_{j \in \mcN(i) } \frac{A_{ij}}{\delta _{ij}}  \quad  \quad \quad \quad \quad \quad \quad \quad  \text{(by definitions of $\beta _{ij}$)}
    \\ &\leq
    2\varepsilon + 
    \frac{4\varepsilon}{\Delta t}  \quad \quad \quad \quad \quad \quad \quad \quad \quad \quad \quad \quad \text{(by the Fourier conditions)}
  \end{align}

\end{proof}

\end{document}